\documentclass[twoside]{article}

\usepackage[accepted]{aistats2022_camera}
\usepackage[round]{natbib}
\usepackage{amsfonts}
\usepackage{amsmath}
\usepackage{amssymb}
\usepackage{amsthm}
\usepackage{graphicx}
\usepackage{graphbox}
\usepackage{xcolor}
\usepackage{verbatim}
\usepackage{comment}

\DeclareMathOperator*{\argmin}{arg\,min}

\newtheorem{theorem}{Theorem}[section]
\newtheorem{lemma}{Lemma}[section]

\RequirePackage{fancyhdr}
\RequirePackage{color}
\RequirePackage{algorithm}
\RequirePackage{algorithmic}
\RequirePackage{eso-pic} % used by \AddToShipoutPicture 
\RequirePackage{forloop}

\newcommand{\C}{\mathbb{C}}

\newtheorem{proposition}{Proposition}
\usepackage[utf8]{inputenc} % allow utf-8 input
\usepackage[T1]{fontenc} % use 8-bit T1 fonts
\usepackage{hyperref} % hyperlinks
\usepackage{url} % simple URL typesetting
\usepackage{booktabs} % professional-quality tables
\usepackage{amsfonts} % blackboard math symbols
\usepackage{nicefrac} % compact symbols for 1/2, etc.
\usepackage{microtype} % microtypography
\usepackage{xcolor} % 

\begin{document}



\twocolumn[

\aistatstitle{Efficient and passive learning of networked dynamical systems driven by non-white exogenous inputs}

\aistatsauthor{ Harish Doddi$^{*1}$ \And Deepjyoti Deka$^{*2}$ \And Saurav Talukdar$^3$ \And Murti Salapaka$^4$ }

\aistatsaddress{ $^1$Department of Mechanical Engineering, University of Minnesota Twin Cities\\
$^2$Theoretical Division T-5, Los Alamos National Laboratory\\ $^3$Google Inc.\\ $^4$Department of Electrical \& Computer Engineering, University of Minnesota Twin Cities\\
$^*$ Equal Contribution} ]

\begin{abstract}
We consider a networked linear dynamical system with $p$ agents/nodes. We study the problem of learning the underlying graph of interactions/dependencies from observations of the nodal trajectories over a time-interval $T$. We present a regularized non-casual consistent estimator for this problem and analyze its sample complexity over two regimes: (a) where the interval $T$ consists of $n$ i.i.d. observation windows of length $T/n$ (restart and record), and (b) where $T$ is one continuous observation window (consecutive). Using the theory of $M$-estimators, we show that the estimator recovers the underlying interactions, in either regime, in a time-interval that is logarithmic in the system size $p$. To the best of our knowledge, this is the first work to analyze the sample complexity of learning linear dynamical systems \emph{driven by unobserved not-white wide-sense stationary (WSS) inputs}. 
\end{abstract}

\section{INTRODUCTION}

A networked linear dynamical system (LDS) is a network of agents/nodes, each of whose state evolves over time (in discrete or continuous steps) as a \emph{linear} function of an external excitation and the states of its neighboring nodes in the network. The framework of LDS has been used to model dynamics in systems biology \citep{porreca2008structural,koh2009using}, financial markets \citep{sandefur1990discrete}, energy \citep{inchauspe2015dynamics}, transportation \citep{stathopoulos2003multivariate} and other critical networks \citep{ascione2013simplified,kroutikova2007state}. Learning the dependencies, or topology learning, in a networked LDS is crucial for inference of influence pathways and subsequent control for the corresponding networks. As such, strategies to recover the underlying network structure from nodal time-series in LDS have been researched and can be classified into two categories: active and passive. Active learning involves efficient manipulation or interventions of nodal dependencies and injecting exogenous inputs into the LDS to infer the edges in the network by identifying the resulting changes \citep{dankers2015errors}. Passive methods, on the other hand, use historical or streaming time-series of nodal states to infer the underlying topology. Our work falls within the domain of passive structure estimation. Very few works discuss learning such systems but are limited to the asymptotic regime (infinite sample limit). Examples include \citet{materassi2012problem,talukdar2015reconstruction,talukdar2020physics}.

\textbf{Prior Work:} Tractable passive topology learning in networked LDS and Vector Auto-Regressive processes (VAR) has been shown using the framework of $l_1$-regularized regression (\citep{basu2015regularized,loh2012high} and references therein), where the focus is on extending the results from the static Lasso or Graphical Lasso \citep{tibshirani1996regression,friedman2008sparse,meinshausen2006high} to one with correlated samples, by showing that properties such as Restricted strong convexity hold. A similar approach for continuous time stochastic differential equation has been studied in \citet{bento2010learning}. A graphical model for VAR processes, without performance guarantees, has been proposed in \citet{songsiri2010graphical}. Least squared regression based identification of unstable dynamical systems using a single trajectory has been studied in \citet{simchowitz2018learning,faradonbeh2018finite}. However, these algorithms rely on the assumption that unobserved exogenous inputs to the system are i.i.d. or white Gaussian noise, or that the exogenous inputs are observed \citep{fattahi2018data,fattahi2019learning}.

\textbf{Temporally correlated inputs:} Learning networks excited by temporally correlated inputs is necessary to extend prior work restricted to learning under i.i.d inputs. Examples of systems excited by colored inputs include power grids, thermal networks of buildings \citet{talukdar2020physics,materassi2010topological}, as well as time-series of air quality, stock market, and magnetoencephalography datasets \citep{dahlhaus2000graphical, tank2015bayesian, bach2004learning}. 

On learning networked LDS with temporally correlated but unobserved inputs, \citet{dahlhaus2000graphical, jung2015graphical, tank2015bayesian} relate the Conditional Independence Graph (CIG) to the support structure of the inverse Power Spectral Density (PSD) of the states. However this is insufficient for true topology recovery as the CIG includes additional edges, \cite{materassi2012problem}. \citet{talukdar2020physics} presents a consistent algorithm for exact recovery in this setting using non-causal regression (Wiener filter), that forms the starting point for the analysis in this article. \citet{quinn2015directed} recovers the underlying topology in networked LDS using the framework of directed mutual information. However, these works do not provide for guarantees in the finite sample regime, aside from numerical examples. 

The overarching goal of this work is thus to provide a structure learning algorithm for networked LDS driven by temporally-correlated inputs, with guarantees on its performance for finite lengths of state trajectories. We present a \emph{regularized Wiener filter estimator} for this problem and determine the observation window $T$ necessary to guarantee correct estimation over two regimes: (a) where $T$ consists of $n$ i.i.d. observation windows of length $T/n$ (restart and record), and (b) where $T$ is one continuous observation window (consecutive).

The rest of the article is organized as follows. In Section \ref{sec:main_results}, we describe the mathematical model of networked linear dynamical system and our consistent learning algorithm. The main results are presented in Theorems \ref{thm:theorem2.1}, \ref{thm:structure} and \ref{thm:theorem2.2}. Section \ref{sec:analysis} contains results on M-estimators used in the proof of our theorems, with sketches of proofs in Section \ref{sec:Main_theorem_proofs}. Section \ref{sec:results} contains simulation results, and Section \ref{sec:conclusions} summarizes the article and includes potential extensions and generalizations. 
\section{MAIN RESULTS} \label{sec:main_results}
Consider a graph $G= (V,E)$ of $p+1$ nodes in set $V = \{1,...,p+1\}$ and undirected edge set $E \subset V\times V$. We denote the set of two-hop neighbors in $G$ by set $E_M$, where $E_M=\{(i,j)| (ij) \in E \text{ or } \exists k, \text{ s.t } (ik), (jk) \in E\}$ (see Fig.~\ref{fig:Illustration}). Note that $E_M \setminus E$ is the set of `strict' two-hop neighbors in the graph $G$, that do not form edges in $E$.
 Each node $i\in V$ is associated with a real-valued scalar state variable $\{x_i(k), k \in \mathbb{Z}\}$ that evolves in discrete time \footnote{we discuss extension to continuous time and higher order models in Section \ref{sec:conclusions}} according to the following linear dynamical equation: 
 \small
\begin{align}\label{eqn:LDM_timedomain}
 x_i(k+1) = h_{ii}x_i(k) + \sum_{(ij) \in E, j\neq i} h_{ij}x_j(k)+ e_i(k), 
\end{align}
\normalsize
where, $\{e_i(k), k \in \mathbb{Z}\}$, is an exogenous input. While samples of $x_i(k)$ are correlated in time due to the system dynamics, prior work on guaranteed learning of networked LDS include only temporally uncorrelated or white excitations/inputs $e_i(k)$. In this work, we consider $e(k)_{k\in\mathbb{Z}} = [e_1(k)... e_{p+1}(k)]^T$ to be a zero-mean \textbf{Wide-Sense Stationary (WSS)} Gaussian process, uncorrelated across nodes, i.e., $\forall k_1, k_2, \tau \in \mathbb{Z}$, $\mathbb{E}[e(k_1)] = \mathbb{E}[e(k_2)] =0$, and $\mathbb{E}[e(k_1+\tau)e(k_1)^T] = \mathbb{E}[e(k_2+\tau)e(k_2)^T].$ The time-series vector $x(k)_{k \in \mathbb{Z}} = [x_1(k)...x_{p+1}(k)]^T \in \mathbb{R}^{p+1}$ is thus a zero mean jointly Gaussian WSS processes. 

The frequency domain representation of Eq.~\ref{eqn:LDM_timedomain} is obtained by taking the Z-transform ($\mathcal{Z}[.]$) on both sides of Eq.~\ref{eqn:LDM_timedomain}. Substituting $z= e^{\iota f}$ for a frequency $f \in [0,2\pi)$, and rearranging for $X_i(f) := \mathcal{Z}[x_i]\rvert_{z= e^{\iota f}}$, we obtain the following:
\small
\begin{align}\label{eqn:LDM_freqdomain}
 X_i(f) = &\sum_{(ij) \in E, j\neq i} H_{ij}(f)X_j(f)+ P_i(f), \text{ where, }\\
 H_{ij}(f) &:=[\mathcal{Z}[h_{ij}](z-\mathcal{Z}[h_{ii}])^{-1}]\rvert_{z= e^{\iota f}}, (ij)\in E, \nonumber \\
 P_i(f) &=[\mathcal{Z}[e_i](z-\mathcal{Z}[h_{ii}])^{-1}]\rvert_{z= e^{\iota f}}. \nonumber
\end{align}
\normalsize
Here, $H_{ij}(f)$ is a linear time-invariant filter. Note that each edge $(ij)\in E$ corresponds to non-zero transfer functions $H_{ij}$ and $H_{ji}$, that may be different. 

Given time-series of $x(k)$, we define the lagged correlation matrix
$R_x(\tau)$ for $\tau\in \mathbb{Z}$, and its Discrete Time Fourier Transform (DTFT), namely, power spectral density $\Phi_x$, at frequency $f$ as 
\small
\begin{align} 
 &R_x(\tau)= \mathbb{E}(x(\tau)x^{T}(0)), \nonumber \\
 &\Phi_x = \mathcal{F}\{R_x(\tau)\} = \lim_{m\rightarrow \infty} \sum_{\tau=-m}^{m} R_x(\tau)e^{-\iota f \tau}. \label{eq:PSD}
\end{align}
\normalsize

\begin{figure}[htb]
 \begin{center}
\includegraphics[width=0.8\columnwidth]{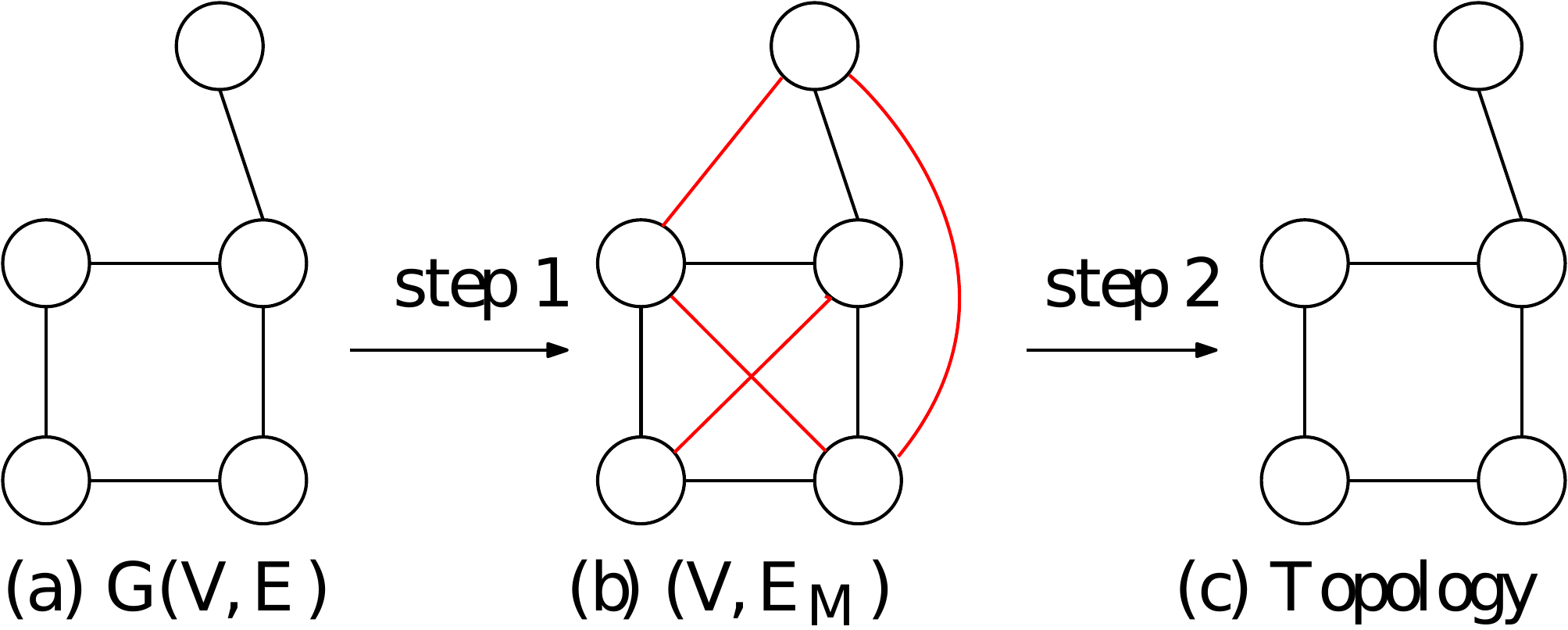} 
\caption{Topology Learning: In step $1$, the two-hop neighborhood set $E_M$ is estimated using Lemma \ref{talukdar_lemma1}(a). In step $2$, strict two-hop neighbors (red colored edges) are eliminated from $E_M$ using Lemma \ref{talukdar_lemma1}(b). 
	\label{fig:Illustration}}
 \end{center}
 \end{figure}
 \textbf{Topology Learning:} Consider $n$ state trajectories of all the nodes in $V$ for the graph $G = (V,E)$ excited by \emph{unobserved} WSS (temporally correlated) inputs, such that the $r^{th}$ state trajectory ($x^r$) has $N$ samples. Let $T=n\times N$ be equal to the total observation window. For the $r^{th}$ state trajectory, define the Discrete Fourier Transform (DFT)\footnote{computed at frequency $f = \frac{2\pi l}{N},~ l \in \{0,\cdots, N-1 \}$ unless explicitly mentioned} is
\small
\begin{align}\label{def:X_i}
 X^r_i &= \frac{1}{\sqrt{N}}\sum_{k=0}^{N-1} x_i^r(k) e^{-\iota f k}, X^r_{\overline{i}} = \frac{1}{\sqrt{N}}\sum_{k=0}^{N-1} x_{\overline{i}}^r(k) e^{-\iota f k}, 
\end{align}
\normalsize
where, $r \in\{1, \cdots,n \}$, $x_{\overline{i}}^r=$ $[x_1^r, \cdots,x_{i-1}^r,x_{i+1}^r$, $ \cdots,x_{p+1}^r]^T$.
Construct $\mathcal{Y} = [X_i^1, \cdots, X_i^n ]^T \in \C^n$ and $\mathcal{X} = [X_{\overline{i}}^1, \cdots, X_{\overline{i}}^n]^T \in \C^{n \times p}$ respectively. We assume that $\mathcal{X}$ and $\mathcal{Y}$ are column-normalized, that is,
\small
\begin{align} \label{def:columnNormalized}
 \frac{\|\mathcal{Y} \|_2}{\sqrt{n}} \leq 1, \quad \frac{\|\mathcal{X}(*,l) \|_2}{\sqrt{n}} \leq 1, \ \forall l \in \{1, \cdots,p\}. 
 \end{align}
 \normalsize
 
Column-normalization is a common data pre-processing step encountered in practice to scale the nodal state trajectories prior to regression. Eq. \ref{def:columnNormalized} is not necessary for Eq. \ref{eqn:LDM_timedomain} to be stable. For any quantity $\beta \in \mathbb{C}$, we use $\Re(\beta)$ and $\Im(\beta)$ to denote its real and imaginary components. 

We list the following result from \citet{talukdar2020physics} that enables consistent estimation of all edges in $E$ (as described in Figure \ref{fig:Illustration}), using nodal state trajectories.

\begin{lemma}[\citet{talukdar2020physics}]\label{talukdar_lemma1}
For $i \in V$ of a well-posed networked LDS, the Wiener filter $W_i$ in Eq.~\ref{eqn:RWF} satisfies (a) $W_i[j] \neq 0$ if and only if $(ij) \in E_M$ (b) for $(ij) \in E_M$, $\Im( {W}_i[j]) \neq 0$ if and only if $(ij)$ is a true edge in $G$.
\small
\begin{align} \label{eqn:RWF}
 W_i = \lim_{n,N \to \infty} \argmin_{\beta \in \mathbb{C}^p} \frac{1}{2n}\| \mathcal{Y} - \mathcal{X}\beta \|_2^2.
\end{align}
\normalsize
\end{lemma}
The proof of Lemma \ref{talukdar_lemma1} (see \citet{talukdar2020physics} for details) follows by showing that $W_i[j]=-[\Phi_x^{-1}(i,i)]^{-1}\Phi_x^{-1}(i,j)$. The result then follows from algebraic properties of $\Phi_x^{-1}$ (inverse power spectral density) derived from Eq.~\ref{eqn:LDM_freqdomain}. It is worth noting that, in the time-domain, Eq.~\ref{eqn:RWF} is equivalent to a \emph{non-causal} regression of the time-series, termed as ``Wiener filter'' \citep{materassi2012problem}. This is effectively a \emph{non-causal extension} of the connection between the inverse covariance matrix and the neighborhood regression used in learning static Gaussian graphical models \citep{friedman2008sparse,meinshausen2006high,ravikumar2008model}.

For the finite sample regime, we study the problem of estimating edges $\hat{E}$ such that $\mathbb{P}[\hat{E} =E] \geq 1-\epsilon$ for any user-defined threshold $\epsilon \in (0,.5)$. Estimating $\Phi_x$ and then inverting it requires significant amount of data in the high dimensional setting. Instead, we use a regularized version of Eq.~\ref{eqn:RWF} as our graph estimator.

\subsection{Regularized Wiener Filter Estimator}
We propose a Regularized Wiener Filter Estimator $\hat{W}_i$ for a node $i \in V$ as follows:
\small
\small\begin{align} \label{eqn:RWFE}
 \hat{W}_i(\lambda) = \argmin_{\beta \in \C^p} \frac{1}{2n}\|\mathcal{Y} - \mathcal{X}\beta \|_2^2 + \lambda \| \beta \|_1,
\end{align}
\normalsize
where, $\lambda >0$ is the regularization parameter. As $\beta \in \mathbb{C}^p$, $\|\beta\|_1$ is equal to the $1,2$-group norm over $[\Re(\beta)\ \Im(\beta)]$. For thresholds $\tau_1,\tau_2$, we construct sets 
\small
\begin{align}\label{eqn:thresholding1}
&\hat{E}_M:=\{(ij)| |\hat{W}_i[j]|+|\hat{W}_j[i]| \geq \tau_1 \},\nonumber \\
&\hat{E}:=\{(ij)|(ij) \in \hat{E}_M, |\Im(\hat{W}_i[j])|+|\Im(\hat{W}_j[i])| \geq \tau_2 \}. 
\end{align}
\normalsize
In the remaining of the article, we find sufficient conditions on $n,\ N$ and $\lambda$ and fix thresholds such that $\mathbb{P}[\hat{E} = E] \geq 1-\epsilon$, for given $\epsilon \in (0,0.5)$. 

We consider two settings for the state trajectories: 

(i) \textbf{Restart \& Record (i.i.d):} The $n$ trajectories of length $N$ each are independent. Here, we start recording and then stop recording after collecting $N$ measurements. For the next trajectory, we restart the recording again with a random state initialization and collect the measurements. Hence, it is a process of restart and record, and $\{x^r\}_{r=1}^N$ are i.i.d. trajectories. 

(ii) \textbf{Consecutive (non i.i.d):} In the second and more realistic setting, we consider the $n$ state trajectories to be consecutive, i.e., $\{x^r\}_{r=1}^N$ correspond to $N$-length intervals from a single larger trajectory of length $n\times N$. See Figure.~\ref{fig:recording} for the two settings considered in this article.

\begin{figure}[htb]
 \begin{center}
\includegraphics[width=.9\columnwidth]{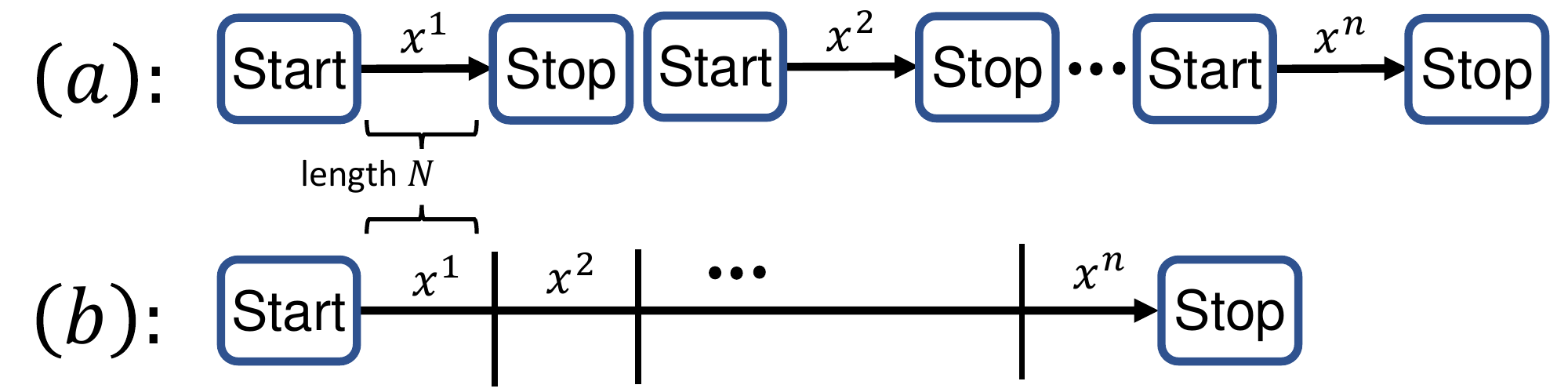} 
\caption{(a) i.i.d trajectories are generated using restart \& record (b) a single trajectory is generated for the non i.i.d, consecutive setting. 
\label{fig:recording}}
\end{center}
\end{figure}

%\textcolor{red}{Deep read till here}
\subsection{Main Results} \label{sec:main_results}
The error in topology learning (see Eq.~\ref{eqn:thresholding1}) arises due to the finite $N$ in computing $\mathcal{X}$ in Eq.~\ref{def:X_i}, as well as the finite $n$ in estimating $W_i$ in Eq.~\ref{eqn:RWFE}. For our analysis, we consider the following non-zero parameters of the LDS over graph $G =(V,E)$.\small
\begin{align} \label{eqn:constants}
L &= \lambda_{min}(\Phi_x^{-1}); \ U = \lambda_{max}(\Phi_x^{-1}); \ d= \max_{i \in V} \deg_{E_M}(i); \nonumber \\
 C&>0,\delta>1,\text{ s.t.}\|R_x(\tau) \|_2 \leq C\delta^{-|\tau|},\ \tau \in \mathbb{Z}; \nonumber \\
 &m_i =\min_{j|(ij)\in E} |\Im(W_i[j])|,\ m = \argmin_{i \in V} m_i.\end{align}
 \normalsize

Note that under persistently exciting inputs, $\Phi_e$ is a positive definite matrix almost surely at all frequencies \citep{materassi2012problem}. Further, $G$ is a connected network. Hence, under standard well-posedness assumptions, $(\mathbb{I}-H)$ and $\Phi^{-1}_P$ in Eq.~\ref{eqn:LDM_freqdomain} are full-ranked and $L\geq 0$. Using norm bounds for matrix products, $L$ and $U$ can be bounded in terms of maximum and minimum eigen-values of $(\mathbb{I} -H^*)(\mathbb{I} -H)$ and $\Phi^{-1}_P$. $C,\delta$ relate to the rate of decay of temporal correlation in the system states. Higher values of $C$ and $\delta^{-1}$ imply greater temporal correlation. $d$, the maximum degree due to edges in $E_M$, is upper-bounded by the square of the maximum nodal degree in $G = (V,E)$. 

The following two theorems bound the errors in estimating $W_i$ by regression (Eq.~\ref{eqn:RWFE}), for restart \& record (i.i.d), and consecutive (non-i.i.d) trajectories respectively. 

\begin{theorem}[\emph{restart \& record}- squared error] \label{thm:theorem2.1}
Let $\epsilon_1 >0$, $i \in V$, $4\sqrt{\frac{3\log(8p/\epsilon_1)}{nL}} \leq \lambda \leq \frac{m_i}{1536 U\sqrt{d}}$, $N \geq \frac{4CU\delta^{-1}}{(1-\delta^{-1})^2}$, and $n \geq max\{\frac{1}{c}\log{\frac{4c'}{\epsilon_1}}, {(3456)^2}(\frac{U}{L}+0.5) \log(2p) d, 3(6144)^2\frac{U^2}{L}d\log(\frac{8p}{\epsilon_1})(\frac{1}{m_i})^2 \}$ where the $n$ trajectories are i.i.d. Then $\| \hat{W}_i(\lambda)-W_i\|_2 \leq \frac{m_i}{2}$ holds with a probability of at least $1-\epsilon_1$. $c,\ c'$ are universal positive constants and $U,\ L,\ C,\ \delta,\ m_i,\ d$ are defined in Eq.~\ref{eqn:constants}. \end{theorem}
\begin{theorem}[\emph{consecutive}- squared error] \label{thm:theorem2.2}
Let $\epsilon_1 >0$ such that $\epsilon_1 \geq \frac{8}{p^2}$ and $i \in V$. 
For $\ 4\sqrt{\frac{(3+ {24\sqrt{3}UC(\delta -1)^{-1}})\log(8p/\epsilon_1)}{nL}} \leq \lambda \leq \frac{m_i}{1536U\sqrt{d}}$, $N \geq \frac{4CU\delta^{-1}}{(1-\delta^{-1})^2}$, and $n \geq \max \{33^2 \log{p} [\frac{U}{L}+0.5+{4\sqrt{8}\frac{CU}{\delta -1}}]^2, 2\log (\frac{8p^2}{p^2 \epsilon_1-8}), (3+ {24\sqrt{3}UC(\delta -1)^{-1}})(6144)^2\frac{U^2}{L}d\log(\frac{8p}{\epsilon_1})(\frac{1}{m_i})^2 \}$, where the $n$ trajectories are non-i.i.d. Then $\| \hat{W}_i(\lambda)-W_i\|_2 \leq \frac{m_i}{2}$ holds with a probability of at least $1-\epsilon_1$, where $c,\ c'$ are universal positive constants. $U,\ L,\ C,\ \delta,\ m_i,\ d$ are defined in Eq.~\ref{eqn:constants}.
\end{theorem}

Using Theorems \ref{thm:theorem2.1} and \ref{thm:theorem2.2}, we give the correctness of the thresholding procedure listed in Eq.~\ref{eqn:thresholding1}.

\begin{theorem}[structure learning] \label{thm:structure}
Let $\epsilon >0$, and $N \geq \frac{4CU\delta^{-1}}{(1-\delta^{-1})^2}$, with constants $U,\ L,\ C,\ \delta,\ m,\ d$ defined in Eq.~\ref{eqn:constants} and universal positive constants $c,c'$. Construct an undirected edge set $\hat{E}_M$ and $\hat{E}$ as per Eq.~\ref{eqn:thresholding1} with thresholds $\tau_1 =\tau_2=m$. Then $E = \hat{E}$ holds with a probability of at least $1- \epsilon$, if\\
(a) `restart \& record' (i.i.d.): $\frac{m}{1536 U\sqrt{d}} \geq \lambda \geq 4\sqrt{\frac{3\log(8p^2/\epsilon)}{nL}}$, and $n \geq max\{\frac{1}{c}\log{\frac{4c'p}{\epsilon}}, {(3456)^2}(\frac{U}{L}+0.5) \log(2p) d, 3(6144)^2\frac{U^2}{L}d\log{\frac{8p^2}{\epsilon}}\frac{1}{m^2} \}$. \\
(b) `consecutive' (non-i.i.d.): $\epsilon \geq \frac{8}{p}$, $\frac{m}{1536 U\sqrt{d}}\geq \lambda \geq 4\sqrt{\frac{(3+ {24\sqrt{3}UC(\delta -1)^{-1}})\log(8p^2/\epsilon)}{nL}}$, and $n \geq max\{ 33^2 \log{p} [\frac{U}{L}+0.5+{4\sqrt{8}\frac{CU}{\delta -1}}]^2, 2\log (\frac{8p^2}{p \epsilon-8}), (3+ {24\sqrt{3}UC(\delta -1)^{-1}})(6144)^2\frac{U^2}{L}d\log(\frac{8p^2}{\epsilon})\frac{1}{m^2} \}$.
\end{theorem}
The proofs of Theorems \ref{thm:theorem2.1}, \ref{thm:theorem2.2} and \ref{thm:structure} are provided in Section \ref{sec:Main_theorem_proofs}. These proofs are based on the theory of M-estimators \cite{negahban2012unified}, for the complex-valued regression problem. It is worth mentioning that when the $n$ trajectories are consecutive, i.e., they correspond to a single time-series, the DFT coefficients computed in Eq.~\ref{def:X_i} are correlated, as against being i.i.d. in the `restart \& record' setting. The derivation of sample complexity in the `consecutive' setting requires concentration results for correlated Gaussian variables, which are more involved and less sharp than comparable results in the i.i.d. setting, as discussed later. In the next section, we present the theory of M-estimators \citep{negahban2012unified} in the complex domain, necessary to prove our results for correct structure recovery.
\section{M-ESTIMATOR BASED ANALYSIS OF REGULARIZED WIENER FILTER} \label{sec:analysis}
The regularized Wiener filter estimator Eq.~\ref{eqn:RWFE} belongs to a class of regularized $M$-estimators. Note that the regularizer ($\|. \|_1$) in Eq.~\ref{eqn:RWFE} satisfies decomposability property with respect to the following complex-valued subspaces: $\mathcal{M}= \{v \in \C^p | v[j] = 0 \text{ if } W_i[j]=0\ \},
\mathcal{M}^{\perp} = \{v \in \C^p | v[j] = 0 \text{ if } W_i[j] \neq 0\ \}$ for a node $i \in V$. That is, $\|v \|_1 = \|v_{\mathcal{M}} \|_1+ \|v_{\mathcal{M}^{\perp}} \|_1$, where $v_{\mathcal{M}}, \ v_{\mathcal{M}^{\perp}} $ are the projections of $v$ on $\mathcal{M}$ and $\mathcal{M}^{\perp}$. We follow the approach in \citet{negahban2012unified} to bound the error
\small
\begin{align} \label{def:Delta}
 \hat{\Delta} := \hat{W}_i(\lambda) - W_i.
\end{align}
\normalsize
\citet{negahban2012unified} states that two conditions are sufficient to control the error $\|\hat{\Delta} \|_2$.
\small
\begin{align} \label{def:condition1}
\noindent\text{{\it{First condition}} ($\lambda$ choice): } \lambda \geq \frac{2}{n}\|\mathcal{X}^H(\mathcal{Y}-\mathcal{X}W_i)\|_\infty.
\end{align}
\normalsize
Eq.~\ref{def:condition1} ensures that $\hat{\Delta}$, defined in Eq.~\ref{def:Delta}, belongs to the set
\small
\begin{align} \label{eqn:D_set}
 \mathcal{D}(W_i) = \{\Delta \in \C^p \ | \|\Delta_{\mathcal{M}^{\perp}} \|_1 \leq 3\|\Delta_{\mathcal{M}} \|_1 \}.
\end{align}
\normalsize
\small
\begin{align}\label{def:condition2}
&\text{{\it{Second condition}} (restricted eigenvalue property): } \nonumber \\ &\frac{1}{n} \|\mathcal{X}\Delta \|_2^2 \geq \kappa \| \Delta \|_2^2, \forall \Delta \in \mathcal{D}(W_i).
\end{align}
\normalsize

\begin{proof}[proof of Eq.~12]
\textbf{Eq.~12} states that if $\lambda \geq \frac{2}{n}\|\mathcal{X}^H(\mathcal{Y}-\mathcal{X}W_i)\|_\infty$ then $\hat{\Delta}:= \hat{W}_i(\lambda) - W_i$ belongs to the set\\ $\mathcal{D}(W_i) = \{\Delta \in \C^p \ | \|\Delta_{\mathcal{M}^{\perp}} \|_1 \leq 3\|\Delta_{\mathcal{M}} \|_1 \}$. 

\textbf{To prove that}, note that \small 
\begin{align}%\label{eqn:Fsimplify}
  & \frac{1}{2n}[\|\mathcal{Y}-\mathcal{X}({W_i} + \Delta)\|_2^2 - \|\mathcal{Y}-\mathcal{X}{W_i}\|_2^2]\nonumber \\
  &= \frac{1}{2n}[(\mathcal{Y}^H -({W_i} + \Delta)^{H}\mathcal{X}^{H})(\mathcal{Y}-\mathcal{X}({W_i} + \Delta)) \nonumber\\
  &~~~~~~~~~~- (\mathcal{Y}^H-W_i^{H}\mathcal{X}^{H})(\mathcal{Y}-\mathcal{X}{W_i})]\nonumber \\
  &=\frac{1}{2n}[\Delta^H\mathcal{X}^H(\mathcal{X}{W_i}-\mathcal{Y}) + (\mathcal{X}{W_i}-\mathcal{Y})^{H}\mathcal{X}\Delta + \Delta^{H}\mathcal{X}^H\mathcal{X}\Delta]\nonumber\\
  &\geq \frac{1}{2n} 2Re(\langle \mathcal{X}^{H}(\mathcal{X}{W_i}-\mathcal{Y}), \Delta \rangle) \nonumber \\
  &\geq \frac{-1}{n}|Re(\langle \mathcal{X}^{H}(\mathcal{X}{W_i}-\mathcal{Y}), \Delta) \rangle|,\nonumber
\end{align}
\normalsize
\noindent where, $Re(x)$ denotes the real part of the complex number $x$. Moreover,
\small 
\begin{align}
\frac{1}{n}|Re(\langle \mathcal{X}^{H}(\mathcal{X}{W_i}-\mathcal{Y}), \Delta \rangle)| &\leq \frac{1}{n}|\langle \mathcal{X}^{H}(\mathcal{X}{W_i}-\mathcal{Y}), \Delta \rangle| \nonumber\\
&\leq \frac{1}{n}\|\mathcal{X}^{H}(\mathcal{X}{W_i}-\mathcal{Y})\|_\infty \|\Delta\|_1 \nonumber\\
&\leq \frac{\lambda}{2}\|\Delta\|_1\nonumber.
\end{align}
\normalsize
Therefore, $\frac{1}{2n}[\|\mathcal{Y}-\mathcal{X}({W_i} + \Delta)\|_2^2 - \|\mathcal{Y}-\mathcal{X}{W_i}\|_2^2]$ $\geq -\frac{\lambda}{2}\|\Delta\|_1=-\frac{\lambda}{2}(\|\Delta_{\mathcal{M}}\|_1+\|\Delta_{\mathcal{M}^\perp}\|_1)$. 

By optimality of $\hat{W}_i(\lambda) = W_i + \hat{\Delta}$ in the Regularized Wiener Filter Estimator, 
\small
\begin{align*}
 0 &\geq \frac{1}{2n}[\|\mathcal{Y}-\mathcal{X}(W_i+\hat{\Delta})\|_2^2-\|\mathcal{Y}-\mathcal{X}{W_i}\|_2^2] \\
 & ~~~~~~~~~+ \lambda[\|W_i + \hat{\Delta}\|_1 - \|{W_i}\|_1]\\
  &\geq -\frac{\lambda}{2}(\|\hat{\Delta}_{\mathcal{M}}\|_1 + \|\hat{\Delta}_{\mathcal{M}^\perp}\|_1) + \lambda(\|\hat{\Delta}_{\mathcal{M}^{\perp}}\|_1 - \|\hat{\Delta}_{\mathcal{M}}\|_1)\\
  &=\frac{\lambda}{2}\|\hat{\Delta}_{\mathcal{M}^\perp}\|_1 - \frac{3\lambda}{2}\|\hat{\Delta}_{\mathcal{M}}\|_1.
\end{align*}
\normalsize
Thus, $\hat{\Delta} \in \mathcal{D}(W_i)$. 
\end{proof}
Next, we show that $\|\hat{W}_i-W_i\|_2 \leq (\frac{3}{\kappa}\lambda \sqrt{d})$, whenever Eq.~11 and Eq.~13 hold. 
The following proposition, similar to Theorem $1$ in \citet{negahban2012unified}, bounds the error $\|\hat{\Delta} \|_2$.
\begin{proposition} \label{prop:Proposition1}
For the regularized Wiener filter estimator defined in Eq.~\ref{eqn:RWFE}, $\|\hat{W}_i-W_i\|_2 \leq (\frac{3}{\kappa}\lambda \sqrt{d})$, whenever Eq.~\ref{def:condition1} and Eq.~\ref{def:condition2} hold.
\end{proposition}

\begin{proof}
Let $K(\delta) := \{ \Delta \in \C^p| \ \Delta \in \mathcal{D}(W_i^f) \text{ and }\|\Delta \|_2= \delta\}. $ Let $\mathsf{F}(\Delta)$ be the difference between the objective of the Regularized Wiener Filter Estimator evaluated at $W_i+ \Delta$ and $W_i$. For a $\Delta \in K(\delta)$, the following holds:
\small 
\begin{align*}
  \mathsf{F}(\Delta) &= \frac{1}{2n}[2Re\langle (\mathcal{X}^H(\mathcal{X}W_i-\mathcal{Y}),\Delta\rangle + \Delta^H\mathcal{X}^H\mathcal{X}\Delta] \\
  &~~~~~~~~~~~~~~+\lambda(\|W_i + \Delta\|_1 - \|W_i\|_1),\\
  &\geq -\frac{1}{n}|Re\langle (\mathcal{X}^H(\mathcal{X}W_i-\mathcal{Y}),\Delta\rangle| + \frac{\kappa }{2}\|\Delta\|_2^2 \\ &~~~~~~~~~~~~+\lambda[\|\Delta_{\mathcal{M}^\perp}\|_1-\|\Delta_{\mathcal{M}}\|_1],\\
  &( \because \text{Using the restricted eigenvalue property) }\\
  &\geq -\frac{\lambda}{2}\|\Delta\|_1 + \frac{\kappa }{2}\|\Delta\|_2^2 + \lambda[\|\Delta_{\mathcal{M}^\perp}\|_1-\|\Delta_{\mathcal{M}}\|_1], \\
  &( \because \text{Using the condition on $\lambda$) }\\
  &=\frac{\kappa }{2}\|\Delta\|_2^2 -\frac{3}{2}\lambda \|\Delta_{\mathcal{M}}\|_1 + \frac{1}{2}\lambda\|\Delta_{\mathcal{M}^\perp}\|_1, \\
  &\geq \frac{\kappa }{2}\|\Delta\|_2^2 -\frac{3}{2}\lambda \|\Delta_{\mathcal{M}}\|_1, \\
  &\geq \frac{\kappa }{2}\|\Delta\|_2^2 -\frac{3}{2}\lambda \sqrt{d} \|\Delta\|_2 \text{ ( \small $ \because \sup_{u \in \mathcal{M}\setminus\{0\}} \frac{\|u \|_1}{\|u \|_2} = \sqrt{d}$ \normalsize),}\\
   &= (\frac{\kappa }{2}\|\Delta\|_2 - \frac{3}{2}\lambda \sqrt{d})\|\Delta\|_2.
\end{align*}
\normalsize
Thus, if $\|\Delta\|_2 = \delta > \frac{3}{\kappa }\lambda \sqrt{d}$, then, $\mathsf{F}(\Delta) > 0$ for all $\Delta \in K(\delta)$. Note that $\mathsf{F}(0) = 0$. Using Lemma $4$ from the Supplementary material of \cite{negahban2012unified} (uses convexity of $\mathsf{F}(\Delta)$), 
it then follows that $\|\hat{\Delta}\|_2 \leq \delta$, that is, $\|\hat{W}_i-W_i\|_2 \leq (\frac{3}{k}\lambda \sqrt{d})$.
\end{proof}

We now show that Eq.~\ref{def:condition1} and Eq.~\ref{def:condition2} hold, for both \emph{restart \& record} (i.i.d.) and \emph{consecutive} (non-i.i.d.) trajectories. These results are then used to prove Theorems \ref{thm:theorem2.1} and \ref{thm:theorem2.2}.

\textbf{Restart \& record (i.i.d.) trajectories:}
\begin{lemma} \label{lem:Lambda}
Suppose $\epsilon_3>0$. Let rows in $\{X_i^r\}_{r=1}^n$ and $\{X_{\overline{i}}^r \}_{r=1}^n$ defined in Eq.~\ref{def:X_i} be i.i.d. If $\lambda \geq 4\sqrt{\frac{3\log(4p/\epsilon_3)}{nL}}$, then $\lambda \geq \frac{2}{n}\|\mathcal{X}^H(\mathcal{Y}-\mathcal{X}W_i)\|_\infty$ holds with a probability of at least $1-\epsilon_3$.
\end{lemma}
\begin{lemma} \label{lem:REP}
Suppose $\epsilon_2>0$ be given. Let rows in $\{X_i^r\}_{r=1}^n$ and $\{X_{\overline{i}}^r\}_{r=1}^n$ defined in Eq.~\ref{def:X_i} be i.i.d. If $n \geq max\{\frac{1}{c}\log{\frac{2c'}{\epsilon_2}}, {(3456)^2}(\frac{U}{L}+0.5) \log(2p) d \}$, $N \geq \frac{4CU\delta^{-1}}{(1-\delta^{-1})^2}$, then Eq.~\ref{def:condition2} holds with $\kappa = \frac{1}{256U}$, with a probability of at least $1-\epsilon_2$.
\end{lemma}
The proofs for Lemmas \ref{lem:Lambda} and \ref{lem:REP} are provided in Section \ref{sec:lambda_proof} and uses concentration bounds for Gaussian random variables.

\textbf{Consecutive (non i.i.d.) trajectories:} 
\begin{lemma} \label{lem:Lambda_nonIID}
Suppose $\epsilon_3>0$. Assume that both $\{X_i^r\}_{r=1}^n$ and $\{X_{\overline{i}}^r\}_{r=1}^n$ defined in Eq.~\ref{def:X_i} are non i.i.d. If $\lambda \geq 4\sqrt{\frac{(3+ {24\sqrt{3}UC(\delta -1)^{-1}})\log(4p/\epsilon_3)}{nL}}$, then $\lambda \geq \frac{2}{n}\|\mathcal{X}^H(\mathcal{Y}-\mathcal{X}W_i)\|_\infty$ holds with a probability of at least $1-\epsilon_3$.
\end{lemma}

\begin{lemma} \label{lem:REP_nonIID}
Suppose $\epsilon_2 >0$ such that $\epsilon_2 \geq \frac{4}{p^2}$. 
Assume that both $\{X_i^r\}_{r=1}^n$ and $\{X_{\overline{i}}^r\}_{r=1}^n$ defined in Eq.~\ref{def:X_i} are non i.i.d.
Then, if $n \geq \max\{33^2 \log{p} [\frac{U}{L}+0.5+{4\sqrt{8}\frac{CU}{\delta -1}}]^2, 2\log (\frac{4p^2}{p^2 \epsilon_2-4})\}$, $N \geq \frac{4CU\delta^{-1}}{(1-\delta^{-1})^2}$, 
then $\frac{1}{n}\|\mathcal{X}\Delta\|_2^2\geq \kappa\|\Delta\|_2^2$, holds for all $\Delta \in \mathcal{D}(W_i)$ with $\kappa = \frac{1}{256U}$, with a probability of at least $1-\epsilon_2$.
\end{lemma}
The proofs for Lemmas \ref{lem:Lambda_nonIID} and \ref{lem:REP_nonIID} are provided in  Section \ref{sec:lambda_proof} and uses concentration bounds for correlated Gaussian random variables.

\section{PROOF OF MAIN THEOREMS}\label{sec:Main_theorem_proofs}
To prove the main theorems for structure learning, we use the M-estimator lemmas from the previous section for the regularized Wiener filter at each node, under both i.i.d. and non-i.i.d. trajectories, and then apply the Union bound for all nodes. 

\begin{proof}[{\it{Proof of Theorem \ref{thm:theorem2.1}}}] 
For $n \geq max\{\frac{1}{c}\log{\frac{4c'}{\epsilon_1}}, {(3456)^2}$ $(\frac{U}{L}+0.5) \log(2p) d {\}}$ and $N \geq \frac{4CU\delta^{-1}}{(1-\delta^{-1})^2}$, we apply Lemma \ref{lem:REP} with $\epsilon_2 = \frac{\epsilon_1}{2}$, then Eq.~\ref{def:condition2} holds with probability of at least $1- \frac{\epsilon_1}{2}$. Here, $\kappa = \frac{1}{256U}$. With $\lambda \geq 4\sqrt{\frac{3\log(8p/\epsilon_1)}{nL}}$, apply Lemma \ref{lem:Lambda} with $\epsilon_3 = \frac{\epsilon_1}{2}$, then Eq.~\ref{def:condition1} holds with probability of at least $1- \frac{\epsilon_1}{2}$. It follows from Proposition \ref{prop:Proposition1} that, $\|\hat{W}_i-W_i\|_2 \leq (\frac{3}{\kappa}\lambda \sqrt{d}) = (768U\lambda \sqrt{d})$. Take $\lambda \leq \frac{m_i}{1536U\sqrt{d}}$. For $n \geq 3(6144)^2\frac{U^2}{L}d\log(\frac{8p}{\epsilon_1})(\frac{1}{m_i})^2 $, $4\sqrt{\frac{3\log(8p/\epsilon_1)}{nL}}$ is smaller than $\frac{m_i}{1536U\sqrt{d}}$. Thus, $\|\hat {W}_i-W_i \|_2 \leq \frac{m_i}{2}$ holds with a probability of at least $1-\epsilon_1$.
\end{proof}

\begin{proof}[{\it{Proof of Theorem \ref{thm:theorem2.2}}}]
Here, we combine the results of Lemma \ref{lem:REP_nonIID} with $\epsilon_2 = \frac{\epsilon_1}{2}$ and Lemma \ref{lem:Lambda_nonIID} with $\epsilon_3 = \frac{\epsilon_1}{2}$. The rest of the proof is analogous to proof of Theorem \ref{thm:theorem2.1}.\\
\end{proof}
\begin{proof}[{\it{Proof of Theorem \ref{thm:structure}}}]
(a) `restart \& record': Choose $\epsilon_1 = \frac{\epsilon}{p}$. It follows from definition of $m$, that $\frac{1}{m} \geq \frac{1}{m_i}$ for all $i \in V$. Now $n \geq max\{\ \frac{1}{c}\log{\frac{4c'p}{\epsilon}}, \ {(3456)^2}(\frac{U}{L}+0.5) \log(2p) d,\ 3({6144})^2\frac{U^2}{L^2}d(\log{\frac{8p^2}{\epsilon}})(\frac{1}{m})^2 \}$ and $4\sqrt{\frac{3\log(8p^2/\epsilon)}{nL}} \leq \lambda \leq \frac{m}{1536 U\sqrt{d}}$ would satisfy the conditions on $n$ and $\lambda$ specified in Theorem \ref{thm:theorem2.1} for a $i \in V$. Therefore, $\| \hat{W}_i-W_i\|_2 \leq \frac{m}{2}$ holds with a probability of at least $1-\frac{\epsilon}{p}$. Using a union bound for all the $p+1$ nodes, we have $\| \Im[\hat {W}_i-W_i] \|_2 \leq \| \hat{W}_i-W_i\|_2 \leq \frac{m}{2}$ holds for all $i \in V$ with a probability of at least $1-\frac{\epsilon(p+1)}{p} \approx 1- \epsilon$ for large $p$.

Note that if $(ij) \in E$, then $|\Im(W_i[j])| \geq m >0$. Similarly, for $(ij) \in E\setminus E_M$, $\Im(W_i[j])=0$ and for $(ij) \notin E_M$, $W_i[j]=0$. Expanding $\| \hat{W}_i-W_i\|_2$, it can thus be shown that $\hat{E}$ derived from $\hat{E}_M$ contains only the edges in $E$.\\

(b) `consecutive': Using Theorem \ref{thm:theorem2.2} for every node $i \in V$ with $\epsilon_1 = \frac{\epsilon}{p}$, the proof is analogous to the proof of Theorem \ref{thm:structure}.\end{proof} 
The next section includes the primary proof techniques for the M-estimator lemmas in Section \ref{sec:analysis}. 

\section{PROOFS OF $M$-ESTIMATOR LEMMAS FOR `RESTART \& RECORD' (I.I.D.) TRAJECTORIES} 
\label{sec:lambda_proof}
The regularized regression in Eq.~\ref{eqn:RWFE} involves working with complex-valued random variables $X_i^r, X_{\overline{i}}^r$ defined in Eq.~\ref{def:X_i} for a node $i \in V$. Their probability distribution is as follows,
\small
\begin{align} \label{def:phiHat}
 \begin{bmatrix} X_i^r \\X_{\overline{i}}^r\end{bmatrix} & \sim \mathcal{N}(\boldsymbol{0},\hat{\Phi}_x), \text{ \normalsize where,\small } \\
 \hat{\Phi}_x =\begin{bmatrix} \hat{\Phi}_i &\hat{\Phi}_{i,\overline{i}}, \\ \hat{\Phi}_{\overline{i},i} & \hat{\Phi}_{\overline{i}} \end{bmatrix}&= \frac{1}{N}\sum_{q=-(N-1)}^{(N-1)}(N-|q|)R_x(q)e^{-\iota f q}. \nonumber
 %\\\hat{\Phi}_i:=\mathbb{E}(X_i^r (X_i^r)^H), \hat{\Phi}_{\overline{i}}:&=\mathbb{E}(X_{\overline{i}}^r (X_{\overline{i}}^r)^H), \hat{\Phi}_{i,\overline{i}}:=\mathbb{E}(X_i^r (X_{\overline{i}}^r)^H), \hat{\Phi}_{\overline{i},i}:=\mathbb{E}(X_{\overline{i}}^r (X_i^r)^H). \nonumber
\end{align}
\normalsize
Thus, $X_i^r \sim \mathcal{N}(0, \hat{\Phi}_i)$ and $X_{\overline{i}}^r \sim \mathcal{N}(\boldsymbol{0}, \hat{\Phi}_{\overline{i}})$. 
The following result bounds the difference between $\hat{\Phi}_x$ and $\Phi_x$ (see Eq.~\ref{eq:PSD}) for a $N$-length trajectory, and is used in our analysis.
\begin{lemma} \label{lem:N_bound}
If $N> \frac{4CU\delta^{-1}}{(1-\delta^{-1})^2}$, then $\| \Phi_x -\hat{\Phi}_x \|_2 \leq \frac{1}{2U}$. Moreover, $\| \Phi_{\overline{i}} -\hat{\Phi}_{\overline{i}} \|_2 \leq \frac{1}{2U},$ and  

$\frac{1}{2U}\| v \|_2^2 \leq v^H \hat{\Phi}_{\overline{i}} v \leq [\frac{1}{L}+\frac{1}{2U}]\| v \|_2^2, \forall v \in \C^p.$

\end{lemma} 

\begin{proof}
From Eq.~9, we have $\frac{1}{U}\leq \lambda_{min}[\Phi_x]$ and $\lambda_{max}[\Phi_x] \leq \frac{1}{L}$. From Eqs.~3, 14, we have
$\|\Phi_x - \hat{\Phi}_x\|_2 $
\small 
\begin{align*}
=& \lim_{m \rightarrow \infty} \|\sum_{p=N}^{m}R_x(p)e^{-\iota f p} + \sum_{p=-m}^{-N}R_x(p)e^{-\iota f p} \\
&~~~~~~~~~~+ \sum_{q=-(N-1)}^{N-1}\frac{|q|}{N}R_x(q)e^{-\iota f q}\|_2,\\
&\leq \lim_{m \rightarrow \infty}\sum_{p=N}^{m}\|R_x(p)\|_2 + \sum_{p=-m}^{-N}\|R_x(p)\|_2 \\ &~~~~~~~~~~~~+\sum_{q=-(N-1)}^{N-1}\frac{|q|}{N}\|R_x(q)\|_2,\\
&\leq \lim_{m \rightarrow \infty}2\sum_{p=N}^{m}C\delta^{-p} +2\sum_{q=1}^{N-1}C\frac{q}{N}\delta^{-q} ~(\because \text{Eq.}~9)\\
&= 2C\delta^{-N}\frac{1}{1-\delta^{-1}}+\frac{2C}{N}[\frac{\delta^{-1}(1-\delta^{-N})}{(1-\delta^{-1})^2} \ -\frac{N\delta^{-N}}{1-\delta^{-1}}]\\
&= \frac{2C}{N}[\frac{\delta^{-1}(1-\delta^{-N})}{(1-\delta^{-1})^2} ]\leq \frac{2C \delta^{-1}}{N(1-\delta^{-1})^2}.\\
\end{align*}
\normalsize
Therefore, if $N> \frac{4CU\delta^{-1}}{(1-\delta^{-1})^2}$, then $\| \Phi_x -\hat{\Phi}_x \|_2 \leq \frac{1}{2U}$. Moreover, for $v \in \C^{p+1}$ we have $\|(\hat{\Phi}_x)^{1/2}v\|_2^2 = v^H [\Phi_x -(\Phi_x-\hat{\Phi}_x) ] v$. It thus follows that, $\|(\hat{\Phi}_x)^{1/2}v\|_2^2 \geq [\frac{1}{U}- \frac{1}{2U}]\|v \|_2^2 = \frac{1}{2U} \|v \|_2^2$, and $v^H \hat{\Phi}_x v \leq [\frac{1}{L}+\frac{1}{2U}]\|v\|_2^2$. Thus, 
\begin{align}
  \frac{1}{2U}\| v \|_2^2 \leq v^H \hat{\Phi}_x v \leq [\frac{1}{L}+\frac{1}{2U}]\| v \|_2^2, \forall v \in \C^{p+1}.
\end{align} Using the same approach, 
\begin{align}
 \frac{1}{2U}\| v \|_2^2 \leq v^H \hat{\Phi}_{\overline{i}} v \leq [\frac{1}{L}+\frac{1}{2U}]\| v \|_2^2, \forall v \in \C^p.\nonumber
\end{align}
\end{proof}

Next we prove Lemma \ref{lem:Lambda}, which gives a lower bound on $\lambda$ used in the regularized Wiener filter estimator. On a high level, the proof uses the Gaussianity of the complex-valued error vector $\mathcal{E}:= \mathcal{Y}-\mathcal{X}W_i$. We use it to identify the Lipschitz constant associated with rows of $\frac{1}{n}\mathcal{X}^H\mathcal{E}$, and then determine the lower bound on $\lambda$ using the union bound.

The following Lemma is useful in the proof of \ref{lem:Lambda}.
\begin{lemma}[covariance (restart \& record)]\label{lem:C1_iid_bound}
Let $\mathcal{E}:= \mathcal{Y}-\mathcal{X}W_i$ with each row corresponding to an i.i.d. trajectory. Let $\mathcal{E}_1:=[\mathcal{E}_R[1]\ \mathcal{E}_I[1] \ ...\ \mathcal{E}_R[n] \ \mathcal{E}_I[n]]^T$ be the re-arranged vector of real and complex entries in $\mathcal{E}$, with covariance matrix $\mathcal{C}_1 = \mathbb{E}[\mathcal{E}_1 \mathcal{E}_1^T]$. Then $\|\mathcal{C}_1 \|_2 \leq \frac{3}{2L}$.
\end{lemma}
\begin{proof}
As the $n$ trajectories are i.i.d., 
\begin{align} \label{def:C1}
\mathcal{C}_1&:= diag(\mathsf{C},...,\mathsf{C}), \text{where} \nonumber \\
&\left(\begin{array}{c} \mathcal{E}_R(j)\\ \mathcal{E}_I(j) \end{array}\right) \normalsize \sim \mathcal{N}(\boldsymbol{0},\mathsf{C}) \Rightarrow \|\mathcal{C}_1\|_2 = \|\mathsf{C}\|_2.
\end{align}

Consider $\mathsf{E}=\Phi_i-\Phi_{i,\overline{i}}W_i - (W_i)^H \Phi_{\overline{i},i}+ (W_i)^H{\Phi_{\overline{i}}}W_i$. Then, $(1+ \|W_i \|_2^2)\frac{1}{U} \leq |\mathsf{E}| \leq (1+ \|W_i \|_2^2)\frac{1}{L}.$ Substituting $W_i = [\Phi_{\overline{i}}]^{-1}\Phi_{i,\overline{i}}$ in $\mathsf{E}$ and using the Schur complement lemma, we get,
\small
\begin{align} \label{eqn:eq4_Lambda}
 |\mathsf{E}|=\Phi_i-\Phi_{i,\overline{i}} (\Phi_{\overline{i}})^{-1}\Phi_{\overline{i},i} = \frac{1}{[\Phi_x]^{-1}(i,i)}.
\end{align}
\normalsize
From the definition of $L,\ U$ in Eq.~9, we have, $\frac{1}{U} \leq |\mathsf{E}| \leq \frac{1}{L}$. Comparing the two inequality bounds of $|\mathsf{E}|$, we get
\begin{align}\label{eqn:eq5_Lambda}
 \frac{L}{U} \leq (1+ \|W_i \|_2^2) \leq \frac{U}{L}.
\end{align}
From the definition of $\mathcal{E}$, it follows that
$\mathbb{E}[\mathcal{E}[j]^H \mathcal{E}[j]] = \hat{\Phi}_i-\hat{\Phi}_{i,\overline{i}}W_i - (W_i)^H \hat{\Phi}_{\overline{i},i}+ (W_i)^H{\hat{\Phi}_{\overline{i}}}W_i$. Define $V = \hat{\Phi}_x - \Phi_x$, and from Lemma 5.1 in the article, we have $\|V \|_2 \leq \frac{1}{2U}$. Substituting $\hat{\Phi}_x = \Phi_x+V$ in $\mathbb{E}[\mathcal{E}[j]^H \mathcal{E}[j]]$, we get, 
\small
\begin{align}\label{def:C_bound}
&\mathbb{E}[\mathcal{E}[j]^H \mathcal{E}[j]]= \frac{1}{[\Phi_x]^{-1}(1,1)}+ V_i - V_{i,\overline{i}}W_i - (W_i)^H V_{\overline{i},i} \nonumber\\
&~~~~~~~~~~~~~~~~~~~~~~~+ (W_i)^H V_{\overline{i}}W_i \ (\because\text{ using Eq.~\ref{eqn:eq4_Lambda}}),\nonumber \\
&\Rightarrow \mathbb{E}(\mathcal{E}_R[j]^2)+ \mathbb{E}(\mathcal{E}_I[j]^2) = Tr(\mathsf{C}) \nonumber\\
&~~~~~~~~~~~~~~~~~~~~~\leq \frac{1}{L}+ (1+ |W_i|_2^2)\|V \|_2 \nonumber\\
&~~~~~~~~~~~~~~~~~~~~~\leq \frac{1}{L}+ \frac{U}{L}\frac{1}{2U}= \frac{3}{2L} \ (\because \text{using Eq.~\ref{eqn:eq5_Lambda}})\nonumber\\
&\Rightarrow\|\mathcal{C}_1\|_2 = \|\mathsf{C}\|_2 \leq \frac{3}{2L}.
\end{align}
\end{proof}

\begin{proof}[{\it{Proof of Lemma \ref{lem:Lambda}}}]
Let $\mathcal{E}:= \mathcal{Y}-\mathcal{X}W_i$. We show that $\frac{1}{n}\|\mathcal{X}^H\mathcal{E}\|_\infty$ is bounded with a high probability and choose $\lambda$ greater than that bound. Separating $\mathcal{X}=\mathcal{X}_R+\iota\mathcal{X}_I,\ \mathcal{E}=\mathcal{E}_R+\iota\mathcal{E}_I$, into real and imaginary parts (specified by subscripts $R$ and $I$ respectively), we have
\small
\begin{align} \label{eqn:eq1_Lambda} 
\frac{1}{n}\|\mathcal{X}^H\mathcal{E}\|_\infty
&\leq \frac{1}{n}\|\left(\begin{array}{c} \mathcal{X}_R^T \ \mathcal{X}_I^T \end{array}\right)\left(\begin{array}{c} \mathcal{E}_R \\ \mathcal{E}_I \end{array}\right)\|_\infty +\nonumber \\ &\frac{1}{n}\|\left(\begin{array}{c} -\mathcal{X}_I^T \ \mathcal{X}_R^T \end{array}\right)\left(\begin{array}{c} \mathcal{E}_R \\ \mathcal{E}_I \end{array}\right)\|_\infty. \end{align}
\normalsize

Let $\mathcal{E}_1:=[\mathcal{E}_R[1]\ \mathcal{E}_I[1] \ ...\ \mathcal{E}_R[n] \ \mathcal{E}_I[n]]^T$ with covariance matrix $\mathcal{C}_1$. Note that \small $\left(\begin{array}{c} \mathcal{E}_R \\ \mathcal{E}_I \end{array}\right). = P\mathcal{E}_1$,\normalsize for some symmetric permutation matrix $P$, such that its covariance matrix $\mathcal{C}_2 = P\mathcal{C}_1P$. Rewriting \small $\left(\begin{array}{c} \mathcal{E}_R \\ \mathcal{E}_I \end{array}\right) = \mathcal{C}_2^{1/2}\left(\begin{array}{c} \mathcal{W}_R \\ \mathcal{W}_I \end{array}\right) $ \normalsize in Eq.~\ref{eqn:eq1_Lambda}, where \small $\left(\begin{array}{c} \mathcal{W}_R \\ \mathcal{W}_I \end{array}\right) \sim \mathcal{N}(\boldsymbol{0},I)$, \normalsize we have 
\small
\begin{align} \label{eqn:eq2_Lambda}
\frac{1}{n}\|\mathcal{X}^H \mathcal{E}\|_\infty &\leq \frac{1}{n}\|\left(\begin{array}{c} \mathcal{X}_R^T \ \mathcal{X}_I^T \end{array}\right)\mathcal{C}_2^{1/2}\left(\begin{array}{c} \mathcal{W}_R \\ \mathcal{W}_I \end{array}\right)\|_\infty +\nonumber \\ &\frac{1}{n}\|\left(\begin{array}{c} -\mathcal{X}_I^T \ \mathcal{X}_R^T \end{array}\right)\mathcal{C}_2^{1/2}\left(\begin{array}{c} \mathcal{W}_R \\ \mathcal{W}_I \end{array}\right)\|_\infty.
\end{align}
\normalsize

To bound the right side of Eq.~\ref{eqn:eq2_Lambda}, we first show that either function is Lipschitz. Consider first \small $f(\mathcal{W}_R,\mathcal{W}_I):=\frac{1}{n}\left(\begin{array}{c} \mathcal{X}_{R}^T(j,:) \ \mathcal{X}_{I}^T(j,:) \end{array}\right)\mathcal{C}_2^{1/2}\left(\begin{array}{c} \mathcal{W}_R \\ \mathcal{W}_I \end{array}\right)$.\normalsize Then, $\|f(\mathcal{W}_R,\mathcal{W}_I) - f(\mathcal{W}_R^{'},\mathcal{W}_I^{'}) \|_2 $
\small
\begin{align}
&\leq \frac{1}{n}\|\left(\begin{array}{c} \mathcal{X}_{R}^T(j,:) \ \mathcal{X}_{I}^T(j,:)\end{array}\right)\|_2\|\mathcal{C}_{2}^{1/2}\|_2 \| \left(\begin{array}{c} \mathcal{W}_R-\mathcal{W}_R^{'} \\ \mathcal{W}_I-\mathcal{W}_I^{'} \end{array}\right)\|_2,\ \label{eqn:eq3_Lambda}\nonumber \\
&\leq \frac{1}{\sqrt{n}}{\|P\|_2}\|\mathcal{C}_1^{1/2}\|_2\|\left(\begin{array}{c} \mathcal{W}_R-\mathcal{W}_R^{'} \\ \mathcal{W}_I-\mathcal{W}_I^{'} \end{array}\right)\|_2, (\because \text{using Eq.~\ref{def:columnNormalized}})\nonumber\\
&= \sqrt{\frac{3}{2nL}}\left(\begin{array}{c} \mathcal{W}_R-\mathcal{W}_R^{'} \\ \mathcal{W}_I-\mathcal{W}_I^{'} \end{array}\right)\|_2 \\
&(\because \text{Lemma \ref{lem:C1_iid_bound}}). \nonumber
\end{align}
\normalsize
Thus, $f$ is a Lipschitz function with Lipschitz constant $\sqrt{\frac{3}{2nL}}$. Using \citet{massart2000some}'s result on concentration of Lipschitz functions, we have, for $t>0$, \small$\mathbb{P}[\frac{1}{n}|\left(\begin{array}{c} \mathcal{X}_{R}^T(j,:) \ \mathcal{X}_{I}^T(j,:) \end{array}\right)\mathcal{C}_2^{1/2}\left(\begin{array}{c} \mathcal{W}_R \\ \mathcal{W}_I \end{array}\right)|\geq t]$ $\leq 2\exp(-\frac{t^2n L}{3})$. \normalsize Choosing $t= \sqrt{\frac{3\log(\frac{4p}{\epsilon_3})}{nL}}$ and the union bound for all $j \in \{1,\cdots,{p}\}$, we have 
$\mathbb{P}[\frac{1}{n}\|\left(\begin{array}{c} \mathcal{X}_R^T \ \mathcal{X}_I^T \end{array}\right)\mathcal{C}_2^{1/2}\left(\begin{array}{c} \mathcal{W}_R \\ \mathcal{W}_I \end{array}\right)\|_\infty \geq \sqrt{\frac{3\log(\frac{4p}{\epsilon_3})}{nL}}]\leq \frac{\epsilon_3}{2}$. Using a similar analysis, 
$\mathbb{P}[\frac{1}{n}\|\left(\begin{array}{c} -\mathcal{X}_I^T \ \mathcal{X}_R^T \end{array}\right)\mathcal{C}_2^{1/2}\left(\begin{array}{c} \mathcal{W}_R \\ \mathcal{W}_I \end{array}\right)\|_\infty \geq \sqrt{\frac{3\log(\frac{4p}{\epsilon_3})}{nL}}]\leq \frac{\epsilon_3}{2}$. Choose $\lambda \geq 4\sqrt{\frac{3\log(4p/\epsilon_3)}{nL}}$. Using the Union bound on Eq.~\ref{eqn:eq2_Lambda}, we have $\mathbb{P}[ \frac{1}{n}\|\mathcal{X}^H \mathcal{E}\|_\infty \geq \lambda/2] \leq \epsilon_3$. %. Then $\lambda \geq \frac{2}{n}\|\mathcal{X}^H(\mathcal{Y}-\mathcal{X}\beta_{*})\|_\infty$ holds with a probability of at least $1-\epsilon_3$.
\end{proof}
Next we prove Lemma \ref{lem:REP} which ensures the restricted eigenvalue property for matrix $\mathcal{X} = [X_{\overline{i}}^1, \cdots, X_{\overline{i}}^n]^T$, where $X_{\overline{i}}^r$ is computed from the $r^{th}$ trajectory, as defined in Eq.~\ref{def:X_i}. On a high level, each row in $\mathcal{X}$ can be divided into real and imaginary components, that are each Gaussian variables with known covariance matrices. The proof then follows by merging bounds on the restricted eigenvalue property of Gaussian real-valued matrices.
\begin{proof}[{\it{Proof of Lemma \ref{lem:REP}}}]
Separating into real and imaginary parts, we have:
\small 
\begin{align} \label{eqn:eq1_REP}
\frac{\|\mathcal{X}\Delta\|_2^2}{{n}} &= \frac{\|(\mathcal{X}_R + \iota \mathcal{X}_I)(\Delta_R + \iota\Delta_I)\|_2^2}{n} = \frac{\|\mathcal{X}_1v\|_2^2}{n}+\frac{\|\mathcal{X}_2v\|_2^2}{n}. 
\end{align}
\normalsize
where $\mathcal{X}_1:= [\mathcal{X}_R \ -\mathcal{X}_I]$, $\mathcal{X}_2:= [\mathcal{X}_I \ \mathcal{X}_R]$ and $v = (\Delta_R^T \ \Delta_I^T)^T$. 

For simplicity, in this proof we drop the superscript $r$ in $X_i^r$ and $X_{\overline{i}}^r$. Note that the rows of $\mathcal{X}_1$ and $\mathcal{X}_2$ are i.i.d. samples of the real random vectors, 
$[(X_{\overline{i}})_R^T \ -(X_{\overline{i}})_I^T]^T$ and $[(X_{\overline{i}})_I^T \ (X_{\overline{i}})_R^T]^T$, respectively. To show that $\frac{1}{n}\|\mathcal{X}\Delta\|_2^2\geq \kappa\|\Delta\|_2^2$ holds for all $\Delta \in \mathcal{D}(W_i)$ with high probability, we prove the restricted eigenvalue property for group structured norms on both terms in Eq.~\ref{eqn:eq1_REP}. Let $\bar{\Sigma} $ be the covariance of random vector $[(X_{\overline{i}})_R^T \ (X_{\overline{i}})_I^T]^T$. Then,
\small
\begin{align}\label{def:SigmaBar}
\overline{\Sigma} = \begin{bmatrix} \overline{\Sigma}_{11} & \overline{\Sigma}_{12} \\ \overline{\Sigma}_{21} & \overline{\Sigma}_{22} \end{bmatrix} = \begin{bmatrix} \mathbb{E}[(X_{\overline{i}})_R (X_{\overline{i}})_R^T] & \mathbb{E}[(X_{\overline{i}})_R (X_{\overline{i}})_I^T] \\ \mathbb{E}[(X_{\overline{i}})_I (X_{\overline{i}})_R^T] & \mathbb{E}[(X_{\overline{i}})_\textit{I} (X_{\overline{i}})_I^T] \end{bmatrix}.
\end{align}
\normalsize
Thus, $[(X_{\overline{i}})_R^T \ -(X_{\overline{i}})_I^T]^T$ and $[(X_{\overline{i}})_I^T \ (X_{\overline{i}})_R^T]^T$ have means $\boldsymbol{0}$ and covariance $\Sigma_1$ and $\Sigma_2$, respectively, where,
% \small
% \begin{align} \label{def:Sigma1}
% \Sigma_1 &= \begin{bmatrix} I & 0 \\ 0 & -I \end{bmatrix}\overline{\Sigma}\begin{bmatrix} I & 0 \\ 0 & -I \end{bmatrix} = \begin{bmatrix} \overline{\Sigma}_{11} & -\overline{\Sigma}_{12} \\ -\overline{\Sigma}_{21} & \overline{\Sigma}_{22} \end{bmatrix}, \nonumber \\
% \Sigma_2 &= \begin{bmatrix} \overline{\Sigma}_{22} & \overline{\Sigma}_{21} \\ \overline{\Sigma}_{12} & \overline{\Sigma}_{11} \end{bmatrix}.
% \end{align}
% \normalsize
\small
\begin{align} \label{def:Sigma1}
\Sigma_1 &= \begin{bmatrix} \overline{\Sigma}_{11} & -\overline{\Sigma}_{12} \\ -\overline{\Sigma}_{21} & \overline{\Sigma}_{22} \end{bmatrix}, \
\Sigma_2 = \begin{bmatrix} \overline{\Sigma}_{22} & \overline{\Sigma}_{21} \\ \overline{\Sigma}_{12} & \overline{\Sigma}_{11} \end{bmatrix}.
\end{align}
\normalsize
From Eq.~\ref{eqn:eq1_REP}, $\|\Delta\|_1 = \|v\|_{1,2} := \sum_{j=1}^{p}\|[v[j] \ v[p+j]]^T\|_2$. Consider the following definitions:
\small $\mathcal{M}_2:=\{c \in \mathbb{R}^{2p}|\ c[i] = 0, c[p+i] = 0,\ \text{if} \ (W_i)_R[i] = 0, (W_i)_I[i]=0 \}$, $\mathcal{M}_2^\perp:=\{c \in \mathbb{R}^{2p}|\ c[i] = 0, c[p+i] = 0,\ \text{if} \ (W_i)_R[i] \neq 0, (W_i)_I[i]\neq 0 \}$ and $\mathcal{D}_2(W_i):=\{v \in \mathbb{R}^{2p}|\ \|v_{\mathcal{M}_2^\perp}\|_{1,2} \leq 3\|v_{\mathcal{M}_2}\|_{1,2} \}$.\normalsize 

Clearly, if $\Delta \in \mathcal{D}(W_i)$, defined in Eq.~\ref{eqn:D_set},
then $v \in \mathcal{D}_2(W_i)$ and vice versa. For $v \in \mathcal{D}_2$, the group norm $\|v\|_{1,2} \leq 4\|v_{\mathcal{M}_2}\|_{1,2}$. Using Cauchy Schwartz inequality and the definition of bounded degree $d$, it follows that, $\|v_{\mathcal{M}_2}\|_{1,2}\leq \sqrt{d} \|v\|_{2}$. Thus, $\|v\|_{1,2}\leq 4\sqrt{d} \|v\|_{2}$. Below is a result, derived from  \citet{negahban2012unified} (Section {$5.1$}) and \citet{ledoux2013probability} for Gaussian random matrices. 
\begin{lemma}
For any Gaussian random matrix $\mathbb{X}\in \mathbb{R}^{n \times 2p}$ with i.i.d. $\mathcal{N}(\boldsymbol{0},\Sigma)$ rows, then there are universal positive constants $c,c^{'}$ such that
with probability at least $1-c^{'}\exp(-cn)$,
\small
 \begin{align} \label{eqn:X_bound}
 &\frac{\|\mathbb{X}v\|_2}{\sqrt{n}} \geq \frac{1}{4}\|\Sigma^{1/2}v\|_2 - \frac{27}{\sqrt{n}}\sqrt{2\log{(2p)}}\rho(\Sigma)\|v\|_{1,2} \end{align}
\normalsize 
where, \small $\rho(\Sigma):=\max_{j\in\{1,...,p\},i\in\{1,2\}}$ $[\mathbb{E}((w_{G_j}(i))^2)]^{1/2}$ and $w \sim \mathcal{N}(0,\Sigma),\ v \in \mathbb{R}^{2p \times 1}.$ 

\normalsize
\end{lemma}
We use Eq.~\ref{eqn:X_bound} with $\mathbb{X}= \mathcal{X}_1$ and $\mathcal{X}_2$ to obtain a lower bound on $\frac{\|\mathcal{X}_1(\Delta_R^T \ \Delta_I^T)^T\|_2}{\sqrt{n}}$ and $\frac{\|\mathcal{X}_2(\Delta_R^T \ \Delta_I^T)^T\|_2}{\sqrt{n}}$. Note that, for $\bar{\Sigma},\Sigma_1$ and $\Sigma_2$ defined in Eqs.~\ref{def:SigmaBar}, \ref{def:Sigma1}, we have $\rho(\bar{\Sigma})= \rho(\Sigma_1)= \rho(\Sigma_2)=\max_j[\bar{\Sigma}_{jj}]^{1/2}$. Using the inequality $\sqrt{a^2+b^2} \leq a+b \leq \sqrt{2(a^2+b^2)}$ for two non-negative numbers $a,b$ in Eq.~\ref{eqn:eq1_REP}, we get that the following holds with a probability of $1 - 2c^{'}\exp(-cn)$, where $c,\ c^{'}$ are universal positive constants, 
\small 
\begin{align} \label{eqn:restricted_final}
&\frac{1}{\sqrt{n}}\|\mathcal{X}\Delta\|_2 \nonumber\\
 &\geq \frac{1}{\sqrt{2}}[\frac{\|\mathcal{X}_1(\Delta_R^T \ \Delta_I^T)^T\|_2}{\sqrt{n}}+ \frac{\|\mathcal{X}_2(\Delta_R^T \ \Delta_I^T)^T\|_2}{\sqrt{n}}]\nonumber \\
 & \geq \frac{1}{\sqrt{2}}[\frac{1}{4}(\|\Sigma_1^{1/2}v\|_2 + \|\Sigma_2^{1/2}v\|_2) - \frac{54}{\sqrt{n}}\sqrt{2\log{(2p)}}\rho(\bar{\Sigma})\|v\|_{1,2}] \nonumber \\
 & \geq \frac{1}{\sqrt{2}}[\frac{1}{4}\sqrt{v^T(\Sigma_1+\Sigma_2)v} - \frac{54}{\sqrt{n}}\sqrt{2\log{(2p)}}\rho(\bar{\Sigma})\|v\|_{1,2}] \nonumber \\
 & \geq \frac{1}{\sqrt{2}}[\frac{1}{4} \lambda_{min}((\Sigma_1+\Sigma_2)^{1/2}) - 
 \frac{54}{\sqrt{n}}\sqrt{2\log{(2p)}}\rho(\bar{\Sigma})4\sqrt{d}]\|v\|_2 ],
\end{align}
\normalsize
\noindent Using Eq.~\ref{def:Sigma1}, the definition of $\hat{\Phi}_{\overline{i}}:=\mathbb{E}(X_{\overline{i}}(X_{\overline{i}})^H)$, for $v := [\Delta_R^T, \Delta_I^T]^T $, it follows that, 
\small
\begin{align}\label{eqn:phiHat_SigmaTilde}
&v^T (\Sigma_1+\Sigma_2)v = \Delta^H \hat{\Phi}_{\overline{i}} \Delta ~~(\because \Sigma_1+\Sigma_2 = \begin{bmatrix} (\hat{\Phi}_{\overline{i}})_R & (\hat{\Phi}_{\overline{i}})_I \\ -(\hat{\Phi}_{\overline{i}})_I & (\hat{\Phi}_{\overline{i}})_R \end{bmatrix}) \nonumber\\
\Rightarrow &\frac{1}{2U}\| \Delta \|_2^2 \leq v^T (\Sigma_1+\Sigma_2) v \leq [\frac{1}{L}+ \frac{1}{2U}]\|\Delta \|_2^2\\ & ~~(\because \text{using Lemma \ref{lem:N_bound}})\nonumber
\end{align}
\normalsize
Thus, $\lambda_{min}((\Sigma_1+\Sigma_2)^{1/2}) \geq \frac{1}{\sqrt{2U}}$, and \small $\rho(\bar{\Sigma}) \leq$ $\rho(\Sigma_1+\Sigma_2)$ $ \leq \|(\Sigma_1+\Sigma_2)^{1/2} \|_2 =\|(\hat{\Phi}_{\overline{i}})^{1/2} \|_2$ $ \leq \sqrt{\frac{1}{L}+\frac{1}{2U}}$\normalsize. Thus, Eq.~\ref{eqn:restricted_final} is,
\small $ \frac{1}{\sqrt{n}}\|\mathcal{X}\Delta\|_2 \geq [\frac{1}{8\sqrt{U}} -$ $\frac{54}{\sqrt{n}}\sqrt{\log{(2p)}}(\sqrt{\frac{1}{L}+\frac{1}{2U}})4\sqrt{d}]\|\Delta \|_2$\normalsize. Choose $n \geq$ $max\{\frac{1}{c}\log{\frac{2c'}{\epsilon_2}}, {(3456)^2}(\frac{U}{L}+0.5) \log(2p) d \}$, then $\frac{54}{\sqrt{n}}\sqrt{\log{(2p)}}(\sqrt{\frac{1}{L}+\frac{1}{2U}})4\sqrt{d} \leq \frac{1}{16\sqrt{U}}. $
Hence, $\frac{\|\mathcal{X}\Delta\|_2}{\sqrt{n}} \geq \frac{1}{16\sqrt{U}}\|\Delta \|_2$, and Eq.~\ref{def:condition2} holds with $\kappa = \frac{1}{256U}$ with a probability of at least $1-\epsilon_2$.
\end{proof}
The proofs of $M$-estimator conditions for consecutive (non-i.i.d.) trajectories (Lemmas \ref{lem:Lambda_nonIID} and \ref{lem:REP_nonIID}) follow on similar lines, albeit with different concentration results.

The following two Lemma's are useful in the proof of \ref{lem:Lambda_nonIID} and \ref{lem:REP_nonIID}.

\begin{lemma}\label{lem:Q1Q2_upperbound}
Suppose $Z := \mathcal{X}\Delta \in \C^n$ with non-i.i.d. rows in $\mathcal{X}$ and $\Delta$. The real component of $Z$ is given by $Z_R= [\mathcal{X}_R \ -\mathcal{X}_I]v$ and the imaginary component is $Z_I= [\mathcal{X}_I \ \mathcal{X}_R]v$, for $v = [\Delta_R^T, \Delta_I^T]^T$. If $N> \frac{4CU\delta^{-1}}{(1-\delta^{-1})^2}$, then $ \|\mathbb{E}[Z_R Z_R^T] \|_2+\|\mathbb{E}[Z_I Z_I^T] \|_2 \leq 2\| \Delta \|_2^2[\frac{1}{L}+\frac{1}{2U} +{4\sqrt{8}\frac{C}{\delta-1}}]$.
\end{lemma}

The proof of the above two Lemma's in Appendix.

\begin{lemma}[covariance (consecutive)]\label{lem:C1_bound}
Let $\mathcal{E}:= \mathcal{Y}-\mathcal{X}W_i$ with non-i.i.d. trajectories per row. Let $\mathcal{E}_1:=[\mathcal{E}_R[1]\ \mathcal{E}_I[1] \ ...\ \mathcal{E}_R[n] \ \mathcal{E}_I[n]]^T$ be the re-arranged vector of real and complex entries in $\mathcal{E}$ with covariance matrix $\mathcal{C}_1 = \mathbb{E}[\mathcal{E}_1 \mathcal{E}_1^T]$ Then $\|\mathcal{C}_1 \|_2 \leq \frac{3}{2L}+ {6\sqrt{3}}\frac{U}{L}\frac{2C}{\delta -1}$.
\end{lemma}

\begin{proof}[{\it{Proof of Lemma \ref{lem:Lambda_nonIID}}}]
The approach for the proof is identical to the proof of Lemma \ref{lem:Lambda} from the article, with few changes. Define $\mathcal{E}:= \mathcal{Y}-\mathcal{X}W_i$. Let $\mathcal{C}_1$ be the covariance matrix of the vector $\mathcal{E}_1:=[\mathcal{E}_R[1]\ \mathcal{E}_I[1] \ ...\ \mathcal{E}_R[n] \ \mathcal{E}_I[n]]^T$. The trajectories aren't independent and $\mathcal{C}_1$ is no more block-diagonal here. An upper bound for $\|\mathcal{C}_1 \|_2$ for this case is provided in Lemma \ref{lem:C1_bound}. Using that, the Lipschitz constant of $f(\mathcal{W}_R, \mathcal{W}_I)$ in Eq.~17 becomes {$\sqrt{\frac{3+ {24\sqrt{3}UC(\delta -1)^{-1}}}{2nL}}$}. Following Lemma \ref{lem:Lambda} of the article, $\lambda \geq 4\sqrt{\frac{(3+ {24\sqrt{3}UC(\delta -1)^{-1}})\log(4p/\epsilon_3)}{nL}}$ gives the result.\end{proof}
\begin{proof}[{\it{Proof of Lemma \ref{lem:REP_nonIID}}}]
Let $Z := \mathcal{X}\Delta \in \C^n$. Its real and imaginary components are $Z_R= [\mathcal{X}_R \ -\mathcal{X}_I]v$ and $Z_I= [\mathcal{X}_I \ \mathcal{X}_R]v$ where $v =[\Delta_R^T \ \Delta_I^T]^T$. We find the lower bounds on $\frac{1}{n}\|Z_R \|_2^2, \frac{1}{n}\|Z_I \|_2^2$ and then combine them to obtain a lower bound of $\frac{1}{n}\|Z \|_2^2 $. Applying Lemma I.2 from the Supplementary material of Negahban and Wainwright (2011) on $Z_R$ and $Z_I$, we have, with individual probability at least $1- [2\exp(-\frac{n(t- \frac{2}{\sqrt{n}})^2}{2})+ 2\exp(-\frac{n}{2})]$ for all $t \geq \frac{2}{\sqrt{n}}$,  
 \small \begin{align}\label{eqn:Z1_bound}
  \frac{1}{n}\|Z_R \|_2^2 &\geq \frac{1}{n}Tr[\mathbb{E}(Z_R Z_R^T)] - 4t\|\mathbb{E}(Z_R Z_R^T) \|_2, \nonumber\\
  \frac{1}{n}\|Z_I \|_2^2 &\geq \frac{1}{n}Tr[\mathbb{E}(Z_I Z_I^T)] - 4t\|\mathbb{E}(Z_I Z_I^T) \|_2.
 \end{align}
 \normalsize
Note that the diagonal values of $\mathbb{E}(Z_R Z_R^T)$ are all equal to $v^T \Sigma_1 v$, and those of $\mathbb{E}(Z_I Z_I^T)$ are equal to $v^T \Sigma_2 v$ with $\Sigma_1, \Sigma_2$ defined in Eq.~20. Using this with Lemma \ref{lem:Q1Q2_upperbound}, we get, $\frac{1}{n}\|Z \|_2^2 = \frac{1}{n}\|Z_R \|_2^2 + \frac{1}{n}\|Z_I \|_2^2 \geq$ 
\small
\begin{align} \label{eqn:Z_bound}
  & v^T(\Sigma_1+\Sigma_2)v- 8t\| \Delta \|_2^2[\frac{1}{L}+\frac{1}{2U} +{4\sqrt{8}\frac{C}{\delta-1}}] \nonumber \\
&\geq \frac{1}{2U}\| \Delta \|_2^2 - 8t\| \Delta \|_2^2[\frac{1}{L}+\frac{1}{2U} +{4\sqrt{8}\frac{C}{\delta-1}}], ~(\because \text{Eq.~23})
\end{align}
\normalsize
holds with probability of at least $1- [4\exp(-\frac{n(t- \frac{2}{\sqrt{n}})^2}{2})+ 4\exp(-\frac{n}{2})]$ for $t \geq \frac{2}{\sqrt{n}}$. Choose $t = \sqrt{\frac{4\log{p}}{n}}$($>>\frac{2}{\sqrt{n}}$ for a large $n$). Then for $n \geq 33^2 \log{p} [\frac{U}{L}+0.5+{4\sqrt{8}\frac{CU}{\delta -1}}]^2$, $\frac{1}{n}\|\mathcal{X}\Delta \|_2^2 \geq \frac{1}{256U}\| \Delta \|_2^2$ holds with probability at least $1- [\frac{4}{p^2}+ 4\exp(-\frac{n}{2})]$. Since $p \geq \sqrt{\frac{4}{\epsilon_2}}$, the statement holds whenever $n \geq 2\log (\frac{4p^2}{p^2 \epsilon_2-4})$. \end{proof}

\section{NUMERICAL RESULTS} \label{sec:results}
We demonstrate the numerical implementation of recovering topology on a Desktop PC with Intel Xeon E5-1620 Processor (8x 3.7 GHz) and 32 GB RAM. We considered a two-dimensional square grid $G$ with $p+1$ nodes and generate samples for Eq.~\ref{eqn:LDM_timedomain} with exogenous input $e(k) = 5s_p[w(k)- 0.3w(k-1)]$, where $w(k),\ w(k-1)$ are sampled from a standard Normal distribution, $s_p \in \mathbb{R}^{(p+1) \times (p+1)}$ is a diagonal matrix containing constants; $h \in \mathbb{R}^{(p+1) \times (p+1)}$ is a weighted adjacency matrix. The scaling $s_p$ is chosen such that $\mathcal{X}$ and $\mathcal{Y}$ in the estimator are column-normalized. 

We reconstruct the topology with a probability of at least $1-\epsilon$, where $\epsilon=0.05$. The numerical experiments are conducted in MATLAB R2020b. 

For a choice of $n$, we generate $n$ trajectories, either independently (restart \& record) or taken as consecutive intervals of a larger trajectory (consecutive). Each trajectory is of length $N = \frac{4CU\delta^{-1}}{(1-\delta^{-1})^2}$, rounded to the nearest integer. From the trajectories, we compute the samples $\{X_i^r,X_{\overline{i}}^r \}_{r=1}^n$ for all the nodes $i \in V$ at frequency $f = \frac{2\pi}{N}$. 

The Regularized Wiener Filter Estimator is solved using CVXR \citep{grant2014cvx}, with $\lambda = 4\sqrt{\frac{3\log(8p^2/\epsilon)}{nL}}$ if trajectories are i.i.d, $\lambda = 4\sqrt{\frac{(3+ {24\sqrt{3}UC(\delta -1)^{-1}})\log(8p^2/\epsilon)}{nL}}$ if trajectories are non-i.i.d. The chosen value of $\lambda$ correspond to the minimum sufficient condition present in Theorem's \ref{thm:structure}, for the i.i.d. and consecutive settings, respectively.

After solving for $\hat{W}_i$, we construct $\hat{E}= \{(i,j) | |\Im(\hat{W}_i[j])|+|\Im(\hat{W}_j[i])|\geq m \}$ ($m$ defined in Eq.~\ref{eqn:constants}). The relative error in reconstructing the topology is defined as the sum of false positive and false negatives. $n_{min}$ is the minimum value of $n$ such that relative error is zero for $45$ out of $45$ random trials. 

The values of $n_{min}$ for various values of $p$ and $\delta$ for the i.i.d. and non-i.i.d. cases are shown in Figure.~\ref{fig:Plot1}. In Figure \ref{fig:Plot1}(b), the correlation strength $\delta^{-1}$ of the trajectories is high, and consequently $N$ is large. For small $\delta^{-1}$, the length of each trajectory can be reduced significantly for reconstructing the topology. For example, in Figure \ref{fig:Plot1}(c), $N$ is much smaller. 
Further, $n_{min}$ is of the order of $\approx 10^7$ rather than a more conservative estimate of $\approx 10^{16}$ as provided by the main theorems.

%Click \href{https://drive.google.com/drive/folders/10Q4ynWUzK16IbTZh-dKpda0ceurdb7Qt?usp=sharing}{here} to download the MATLAB code.

\begin{figure}[!htb]
 \begin{center}
 \begin{tabular}{lll}
 \begin{tabular}{c}(a)\\\\\end{tabular}~~~~~~~~\includegraphics[width=0.3\columnwidth]{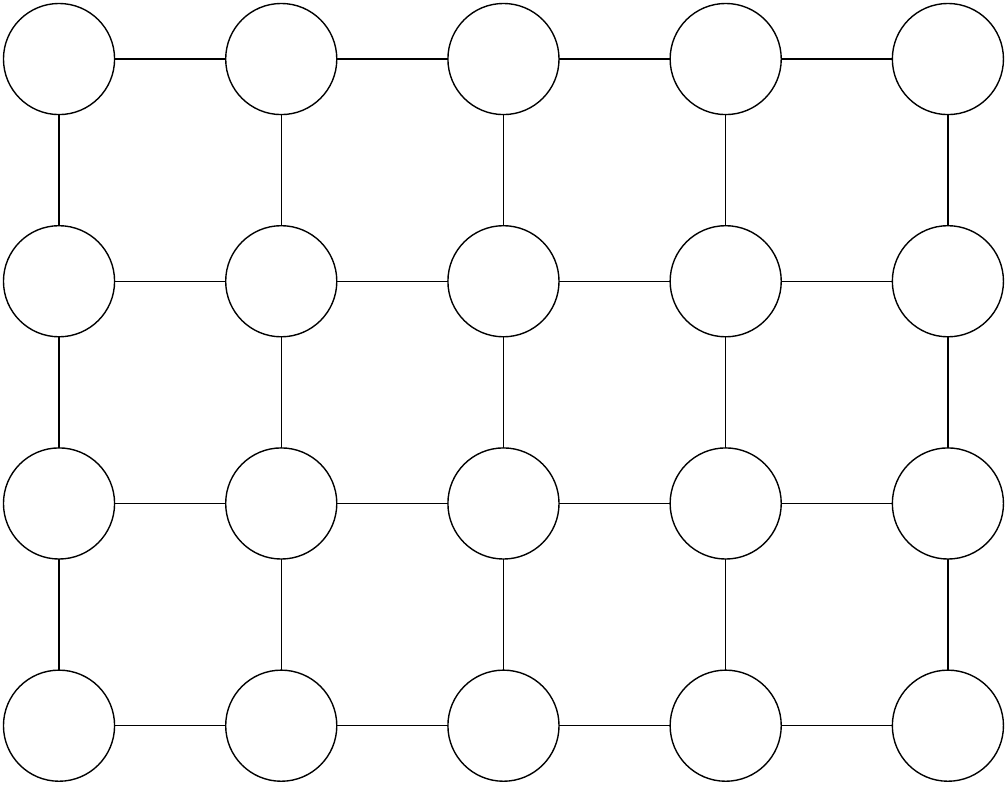}\\
 \begin{tabular}{c}(b)\\\\\\\end{tabular}~~~\includegraphics[width=0.7\columnwidth]{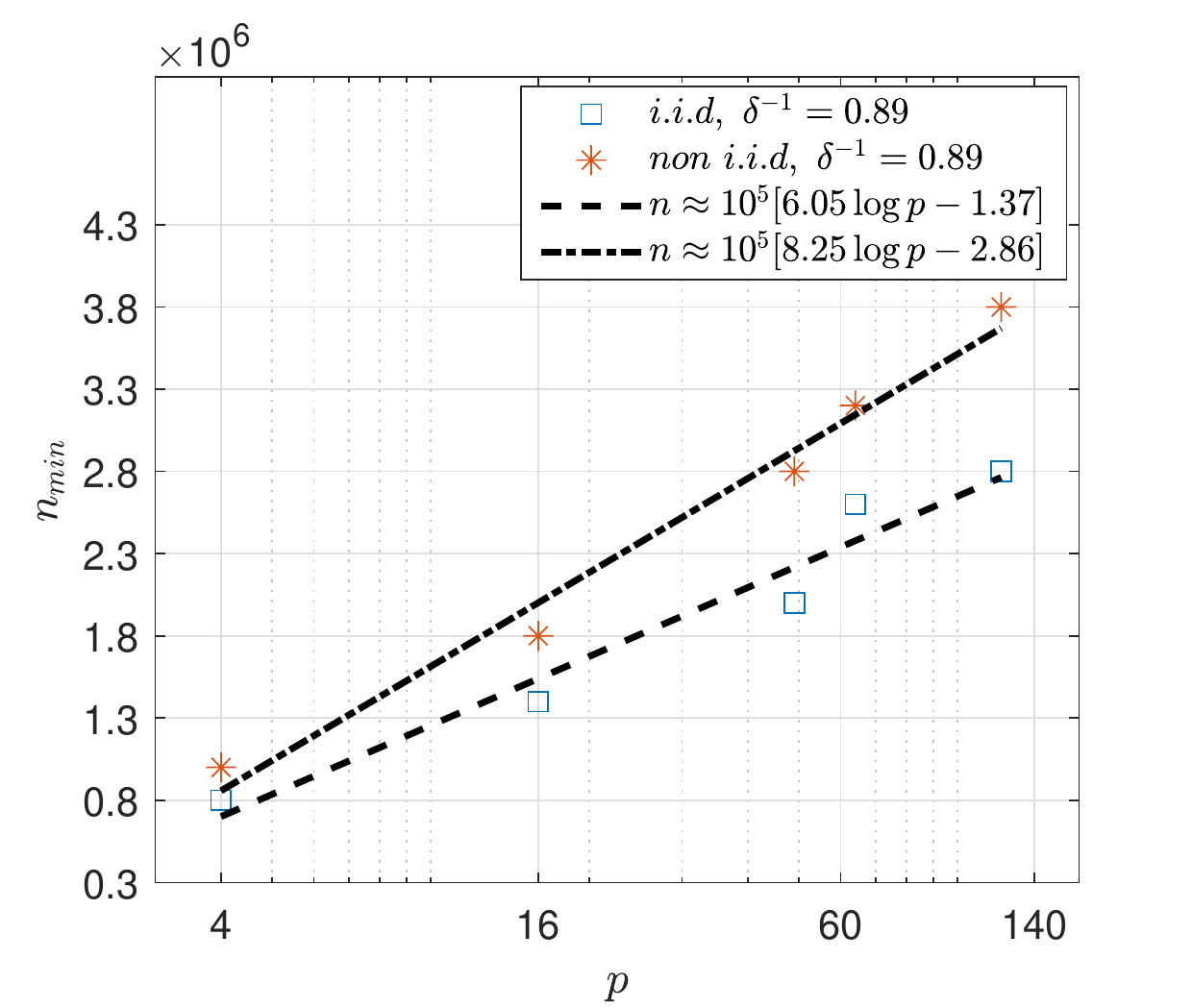}\\
 \begin{tabular}{c}(c)\\\\\\\end{tabular}~~~\includegraphics[width=0.7\columnwidth]{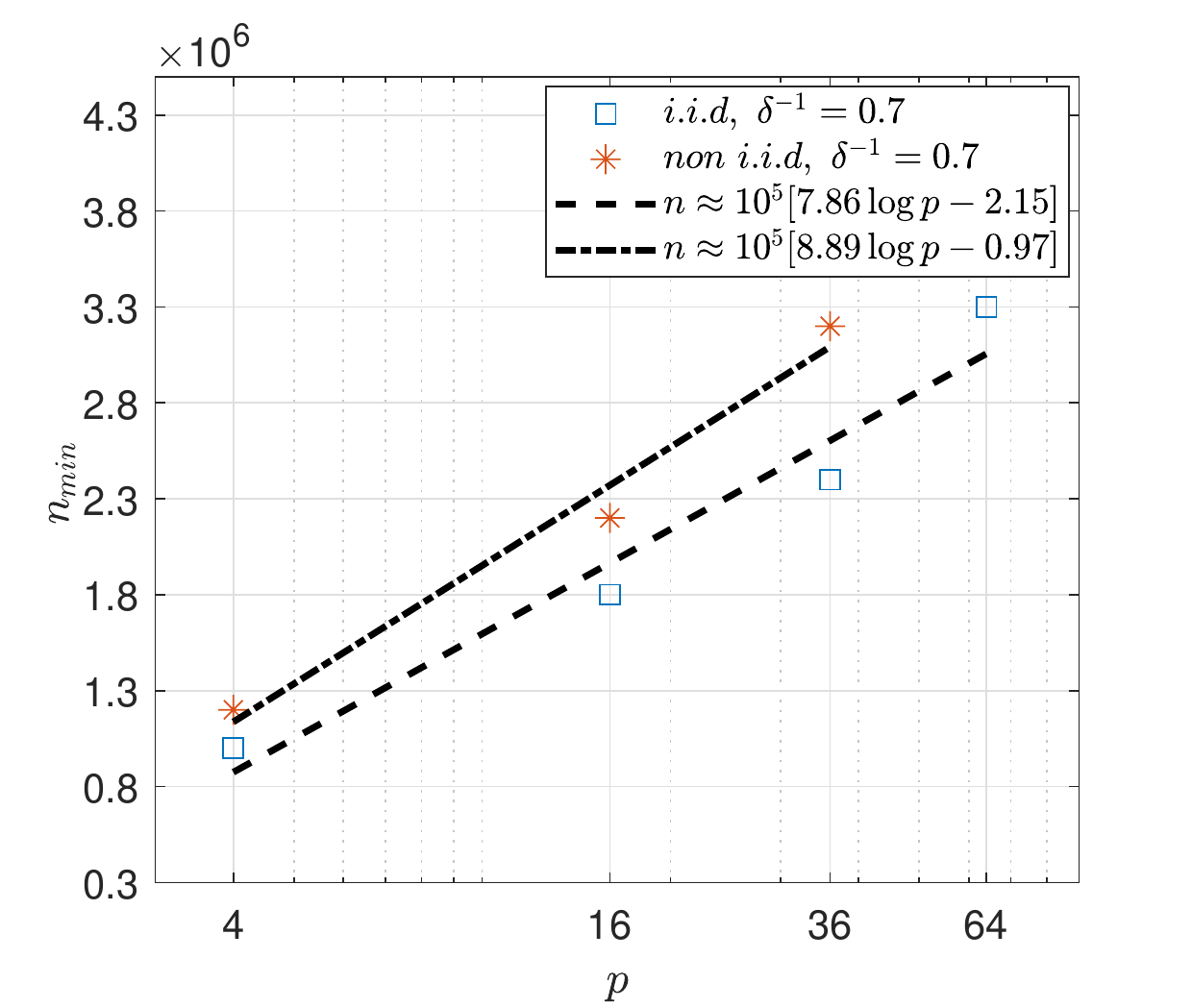}\\
\end{tabular}
\caption{(a) An illustration of $G$ as a $5 \times 5$ grid with $p = 24$. Dependence of $n_{min}$ on $\log p:$ (a) $C = 6.8$, $\delta^{-1} = 0.89$, $U = 1.55$, $L = 0.74$,
$N =\frac{4CU\delta^{-1}}{(1-\delta^{-1})^2}\approx 2900$, (b) $C = 2.8$, $\delta^{-1} = 0.7$, $U = 1.3$, $L = 0.8$, $N =\frac{4CU\delta^{-1}}{(1-\delta^{-1})^2}\approx 115$. Dashed lines in (a) and (b) corresponds to least squares regression fit.
\label{fig:Plot1}}
 \end{center}
 \end{figure}

%Let $n_{min}$ be the minimum number of samples required for learning the topology successfully over $10$ random trials, for each network-size. We verify the dependence of $n_{min}$ on $\log p$. The procedure for obtaining $n_{min}$ is as follows: (a) for a fixed $N$, start with an initial value of $n$, (b) solve Eq. \ref{eqn:RWFE} for a choice of $\lambda = 4\sqrt{\frac{3\log(8p^2/\epsilon)}{nL}}$, with $\epsilon= 0.05$, $L = 0.32$, (b) calculate $\hat{E}$ as shown in Eq. \ref{eqn:thresholding1}, (c) increment the value of $n$ in steps of $20000$ until it attains $n_{min}$ such that $\hat{E} =E$. For the two dimensional square grid considered, we computed $n_{min}$ for $(p+1)\in \{2^2, 4^2, 5^2, 6^2, 7^2, 8^2, 11^2 \}.$ Figure \ref{fig:n_logp_IID}(b) shows the linear dependence of $n_{min}$ on $\log p$ when the $n$ trajectories are independent. By liner fitting, we obtain the relation $n \approx 10000[7.615 \log p - 0.12]$, which is lower than the conservative theoretical bounds given in Theorem \ref{thm:structure}. Additional details of our experiments and plots for continuous dependant trajectories are given in the Appendix.

{\textbf{Numerical Comparison with prior work:} We give empirical comparison with frequency-domain based gLasso-estimator in \citet{jung2015graphical} and unregularized regression in \citet{talukdar2020physics} for a two-dimensional square grid containing 16 nodes. Note that \citet{jung2015graphical} does not lead to correct recovery as Conditional Independence Graph (CIG) doesn't lead to true underlying network. Unlike \citet{talukdar2020physics}, regularization in our algorithm gives improved exact topology recovery in low sample regime. Figure \ref{fig:Jung_comparision} shows the relative error for different values of $n$. The error is computed by averaging over $200$ random trials for each algorithm. Further, for $\epsilon >0$, the fraction of trials with successful topology reconstruction in our experiments is higher than $1-\epsilon.$}
 
\begin{figure}[htb]
 \begin{center}
\includegraphics[width=0.7\columnwidth]{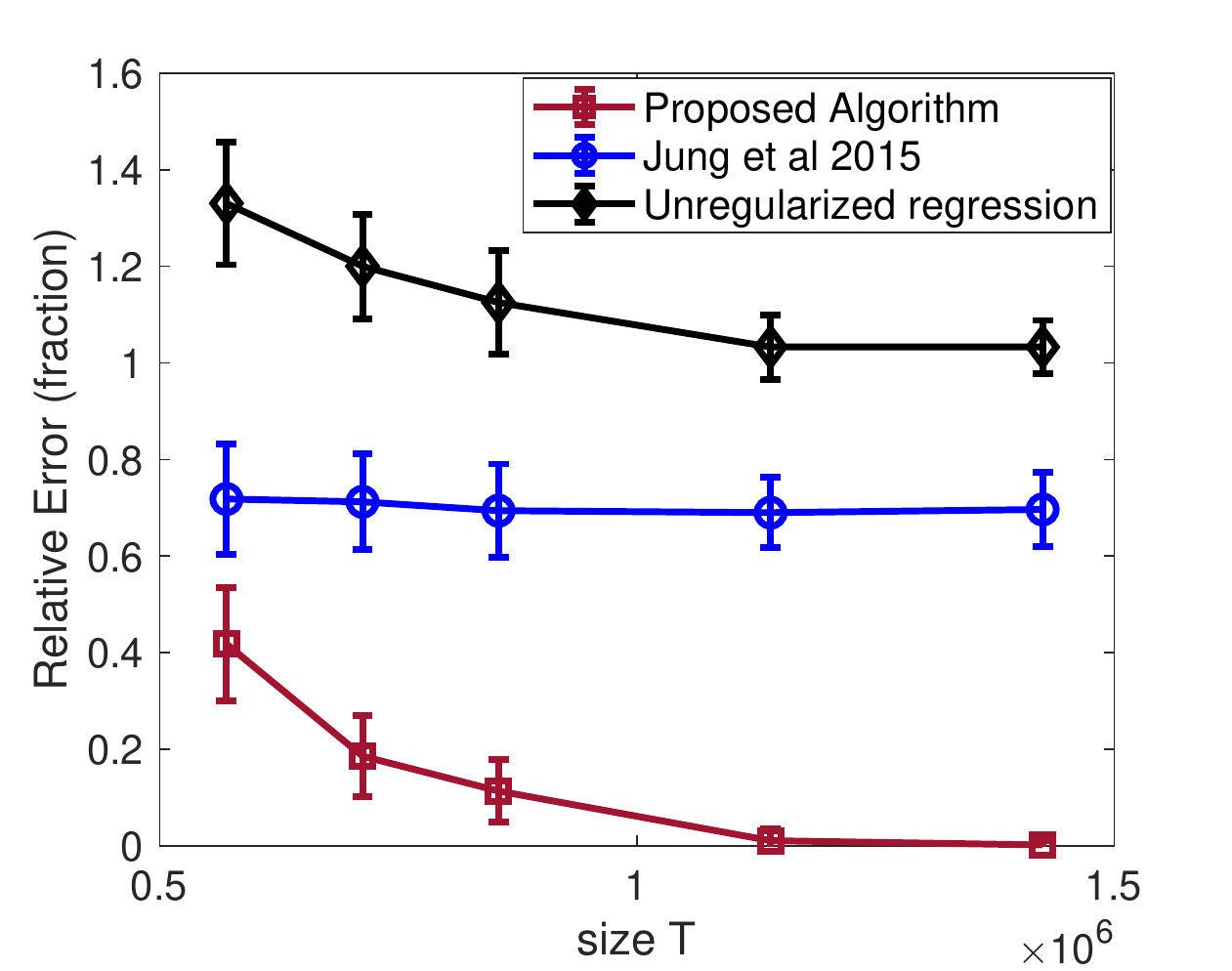} \caption{Reconstruction of exact topology (Proposed Algorithm) vs CIG (\citep{jung2015graphical}).
\label{fig:Jung_comparision}}
\end{center}
\end{figure}

 \section{EXTENSIONS AND PATH FORWARD}
\label{sec:conclusions}
In this article, we presented a regularized Wiener filter estimator to learn the structure of a discrete-time networked LDS. We analyzed the sample complexity of our estimator and showed that it linearly depends the logarithm of the number of nodes $p$ in two cases, one where trajectories of nodal states are collected in independent observation windows of equal length, and another where the trajectories pertain to a single continuous observation window. 

While we discuss our method for first-order discrete-time LDS, our estimator can be extended to learning related networks as highlighted next. 

\textbf{VAR($\tau$) models with correlated inputs:} Lemma \ref{talukdar_lemma1} and our subsequent analysis follows directly if higher-order delays (at the same node) are included in the LDS Eq.~\ref{eqn:LDM_timedomain}. $L,U$ and the sample complexity will need to be changed accordingly.\\
\textbf{Continuous time LDS:} Considering a fixed sampling time $\Delta T$ and a time-discretization function, the continuous time LDS can be converted to a discrete-time LDS with related frequency domain-representation \citep{talukdar2020physics}. The analysis will involve merging the error due to discretization with the finite sample analysis.\\ \textbf{LDS under cyclo-stationary processes:} Cyclo-stationary processes represent a generalization of WSS processes where the statistics such as mean, correlation function are periodic functions of time. As shown in \citet{doddi2019exact}, a lifting operation can be used to represent time-evolution of a cyclo-stationary process as a WSS process with vector-valued states. The remaining analysis of the sample complexity will be similar.\\
\textbf{Directed graphs under correlated inputs:} Note that our estimator uses properties of the inverse power spectral density (see Lemma 2.1 and discussion). Under a strict causality assumption on the linear filters \citep{materassi2012problem, quinn2015directed}, it has been shown that directed edges can be recovered using inverse power spectral density. The framework presented here can thus be extended to efficiently learn a family of directed graphs.

Finally, we plan to analyze the restrictions of our algorithm in learning networked LDS with spatially correlated inputs, and estimating directed networks with non-causal dependencies, where only approximate reconstruction may be possible using passive methods.

\textbf{Non-linear interactions:}
While the theory and validating experiments are conducted for linear dynamics, we claim that the results will follow also for non-linear network dynamics that are approximately linear around an operating point. Such experiments have been described for thermal network of buildings in \citet{talukdar2020physics}. 

\section{Acknowledgements:}
The authors acknowledge support from the Center for Non-Linear Studies (CNLS) and the Information Science and Technology Institute (ISTI) at Los Alamos National Laboratory.

\newpage
\onecolumn
\textbf{Appendix:}
\begin{proof}[\textit{proof of Lemma \ref{lem:Q1Q2_upperbound}}]
%For a $l \in \{1,\cdots,n \}$, let $\mathcal{Y}[l]$ be the $l^{th}$ sample of $X_i$ computed using from $\{x_i((l-1)N), \cdots, x_i((l-1)N+N-1) \}$. Then, $\mathcal{Y}[1] = \frac{1}{\sqrt{N}} \sum_{t=0}^{N-1}x_i(t)e^{-\iota f t}$ and $\mathcal{Y}[2] = \frac{1}{\sqrt{N}} \sum_{s=0}^{N-1}x_i(s+N)e^{-\iota f s}$. 
For $r \in \{1,\cdots,n \}$, $X^r_i$ ($r^{th}$ entry of $\mathcal{Y}$) and $(X_{\overline{i}}^r)^T$ ($r^{th}$ row of $\mathcal{X}$) are computed using $N$ consecutive samples $\{x((r-1)N), \cdots, x((r-1)N+N-1) \}$, using Eq.~4. For $r,\ c \in \{1,\cdots,n \}$ we have,
 \small
 \begin{align} 
 |\mathbb{E}[Z_R Z_R^T](r,c)| &= |v^T \mathbb{E}\begin{bmatrix} [\mathcal{X}_R(r,:)]^T\mathcal{X}_R(c,:) & -[\mathcal{X}_R(r,:)]^T\mathcal{X}_I(c,:) \\ -[\mathcal{X}_I(r,:)]^T\mathcal{X}_R(c,:) & [\mathcal{X}_I(r,:)]^T\mathcal{X}_I(c,:) \end{bmatrix}v|, \nonumber \\
 \Rightarrow |\mathbb{E}[Z_R Z_R^T](r,c)| &\leq \| \Delta \|_2^2\|\mathbb{E}\begin{bmatrix} [\mathcal{X}_R(r,:)]^T\mathcal{X}_R(c,:) & -[\mathcal{X}_R(r,:)]^T\mathcal{X}_I(c,:) \\ -[\mathcal{X}_I(r,:)]^T\mathcal{X}_R(c,:) & [\mathcal{X}_I(r,:)]^T\mathcal{X}_I(c,:) \end{bmatrix} \|_2.\label{eqn:Q1ij_old}\\
 \text{Similarly,~} |\mathbb{E}[Z_I Z_I^T](r,c)| 
 &\leq \| \Delta \|_2^2\|\mathbb{E}\begin{bmatrix} [\mathcal{X}_I(r,:)]^T\mathcal{X}_I(c,:) & [\mathcal{X}_I(r,:)]^T\mathcal{X}_R(c,:) \\ [\mathcal{X}_R(r,:)]^T\mathcal{X}_I(c,:) & [\mathcal{X}_R(r,:)]^T\mathcal{X}_R(c,:) \end{bmatrix} \|_2.\label{eqn:Q2ij_old}
\end{align}
\normalsize 
Consider $\mathcal{Y}[1] = \frac{1}{\sqrt{N}} \sum_{t=0}^{N-1}x_i(t)e^{-\iota f t}$ and $\mathcal{Y}[2] = \frac{1}{\sqrt{N}} \sum_{s=0}^{N-1}x_i(s+N)e^{-\iota f s}.$ The correlation between $\mathcal{Y}_R[1]$ and $\mathcal{Y}_R[2]$ is given by
\small
  \begin{align*}
    |\mathbb{E}[\mathcal{Y}_R[1]\mathcal{Y}_R[2]]| &= |\mathbb{E}[\frac{1}{N} \sum_{t=0}^{N-1} \sum_{s=0}^{N-1} x_i(t)x_i(s+N)\cos(ft)\cos(fs)]|\\
    &= |\frac{1}{N} \sum_{t=0}^{N-1} \sum_{s=0}^{N-1} \mathbb{E}[x_i(t)x_i(s+N)]\cos(ft)\cos(fs)|\\
    &\leq \frac{1}{N} \sum_{t=0}^{N-1} \sum_{s=0}^{N-1} |R_i(t-s-N) |= \frac{1}{N} \sum_{q=-(N-1)}^{N-1} (N-|q |) |R_i(q-N)|.\\
  \end{align*}
  \normalsize
Expanding in time-domain (see Lemma \ref{lem:Q1Q2_upperbound}'s proof), it can be shown that
 \small
 \begin{align} \label{eqn:RR_correlation}
 & \|\mathbb{E}[\mathcal{X}_R(r,:)^T\mathcal{X}_R(c,:)]\|_2\leq B^{rc}, \text{~~where~~} B^{rc}:= \frac{1}{N} \sum_{q=-(N-1)}^{N-1} (N-|q |) \|R_x(q+(r-c)N)\|_2. \end{align}
 \normalsize 
 Similarly, $\|\mathbb{E}[\mathcal{X}_R(r,:)^T\mathcal{X}_I(c,:)]\|_2$, 
 $\|\mathbb{E}[\mathcal{X}_I(r,:)^T\mathcal{X}_I(c,:)]\|_2$, 
 $\|\mathbb{E}[\mathcal{Y}_R[r]\mathcal{X}_I(c,:)]\|_2$, 
 $\|\mathbb{E}[\mathcal{Y}_R[r]\mathcal{X}_R(c,:)]\|_2$, 
 $\|\mathbb{E}[\mathcal{Y}_R[r]\mathcal{Y}_I[c]]\|_2$ and 
 $\|\mathbb{E}[\mathcal{Y}_R[r]\mathcal{Y}_R[c]]\|_2$ are each upper bounded by $B^{rc}$. From Eqs.~\ref{eqn:Q1ij_old}, ~\ref{eqn:Q2ij_old}, it follows that, $|\mathbb{E}[Z_R Z_R^T](r,c)| \leq \| \Delta \|_2^2\sqrt{8}B^{rc}$ and $|\mathbb{E}[Z_I Z_I^T](r,c)| \leq \| \Delta \|_2^2\sqrt{8}B^{rc}$. Thus, 
\small
 \begin{align} \label{eqn:Q1Q2_bound}
 \|\mathbb{E}[Z_R Z_R^T] \|_2+\|\mathbb{E}[Z_I Z_I^T] \|_2 &\leq \max_{r=1}^n \sum_{c=1}^n|\mathbb{E}[Z_R Z_R^T](r,c)|+\max_{r=1}^n \sum_{c=1}^n|\mathbb{E}[Z_I Z_I^T](r,c)| \nonumber\\
 &\leq 2\max_{r=1}^n \sum_{c=1}^n\ (|\mathbb{E}[Z_R Z_R^T](r,c)|+|\mathbb{E}[Z_I Z_I^T](r,c)|)\nonumber\\
 & = 2[v^T(\Sigma_1+\Sigma_2)v + \max_{r} \sum_{c=1,c\neq r}^n\ (|\mathbb{E}[Z_R Z_R^T](r,c)|+|\mathbb{E}[Z_I Z_I^T](r,c)|)\nonumber\\
 & \leq 2\| \Delta \|_2^2[(\frac{1}{L}+\frac{1}{2U}) + \max_r \sum_{c, c \neq r}2\sqrt{8}B^{rc} ].
 \end{align}
\normalsize
In the remaining, we find an upper bound for $\max_r \sum_{c, c \neq r}B^{rc} $. From Eq.~\ref{eqn:RR_correlation},

\small
\begin{align*}
 B^{rc}& \leq \frac{1}{N}\sum_{q= -(N-1)}^{N-1}(N-|q|)C \delta^{-|q-cN+rN|}, \ \ (\because \|R_x(\tau) \|_2 \leq C\delta^{-|\tau|} \text{ from Eq.~9}),\\
 & = \frac{C}{N}[ \sum_{q= 1}^{N-1}(N-q) (\delta^{-|-q-cN+rN|}+ \delta^{-|q-cN+rN|})]+ C\delta^{-|-cN+rN|},\\
 & = C\delta^{-|r-c|N}[ \sum_{q= 1}^{N-1}(1-\frac{q}{N})(\delta^q+ \delta^{-q})+1]\\
 & = C\delta^{-|r-c|N}[1+S_a+S_b- \frac{S_c}{N} - \frac{S_d}{N}],\\
\text{where},~& S_a= \sum_{q=1}^{N-1}\delta^q = \frac{\delta^N-\delta}{\delta -1},~S_c = \sum_{q=1}^{N-1}q \delta^q = \frac{\delta - \delta^N}{(\delta-1)^2}+ \frac{(N-1)\delta^N}{\delta-1} ,\\
 & S_b= \sum_{q=1}^{N-1}\delta^{-q} = \frac{1-\delta^{-(N-1)}}{\delta -1}, ~S_d = \sum_{q=1}^{N-1}q \delta^{-q} =\frac{\delta^{-1}-\delta^{-N}}{(1-\delta^{-1})^2}-\frac{(N-1)\delta^{-N}}{1-\delta^{-1}}.
\end{align*}
{
\begin{align} \label{eqn:gamma_sumBound1}
\text{Thus,} \max_r \sum_{c=1,c\neq r}^{n} B^{rc} &\leq C[1+S_a+S_b - \frac{S_c}{N} - \frac{S_d}{N}]\sum_{c=1,c\neq r}^{n}\delta^{-|r-c|N} \nonumber\\
 &\leq C[1+S_a+S_b - \frac{S_c}{N} - \frac{S_d}{N}]\sum_{c=1}^{\infty}2\delta^{-cN}\nonumber\\
 &\leq C[1+S_a+S_b - \frac{S_c}{N} - \frac{S_d}{N}][\frac{2\delta^{-N}}{1-\delta^{-N}}],\nonumber\\
 &=C [1+S_a + \frac{S_c}{N}(\delta^{-N}-1)][\frac{2\delta^{-N}}{1-\delta^{-N}}] \ (\because \text{$\delta^N(S_b-\frac{S_d}{N})=\frac{S_c}{N}$}),\nonumber\\
 &\leq C(1+S_a )\frac{2\delta^{-N}}{1-\delta^{-N}} ~~(\because \delta^{-N}-1 \text{ is negative and can be ignored}), \nonumber \\ 
 &\leq C(1+\frac{\delta^{N}-\delta}{\delta-1})\frac{2\delta^{-N}}{1-\delta^{-N}} = C (\frac{\delta^{N}-1}{\delta-1})(\frac{2}{\delta^N-1}) =  \frac{2C}{\delta-1}.
\end{align}}

Substituting Eq.~\ref{eqn:gamma_sumBound1} in Eq.~\ref{eqn:Q1Q2_bound} gives $\|\mathbb{E}[Z_R Z_R^T] \|_2+\|\mathbb{E}[Z_I Z_I^T] \|_2 \leq 2\| \Delta \|_2^2[\frac{1}{L}+\frac{1}{2U} +{4\sqrt{8}\frac{C}{\delta-1}}]$.
\end{proof}

\begin{proof}[\textit{proof of Lemma \ref{lem:C1_bound}}]
Writing real and imaginary parts of $\mathcal{E}_1$ and using inequality of matrix $2$ and $\infty$-norms, we have,
 \small
 \begin{align} \label{eqn:C1_upperbound}
 \|\mathcal{C}_1 \|_2 &\leq \max \left[\max_{r=1}^n \sum_{c=1}^{n} \left( |\mathbb{E}[\mathcal{E}_R\mathcal{E}_R^T](r,c)| + |\mathbb{E}[\mathcal{E}_R\mathcal{E}_I^T](r,c)| \right), \max_{r=1}^n \sum_{c=1}^{n} \left( |\mathbb{E}[\mathcal{E}_I\mathcal{E}_R^T](r,c)| + |\mathbb{E}[\mathcal{E}_I\mathcal{E}_I^T](r,c)|\right) \right],\nonumber \\
 &\leq \max \begin{bmatrix}
 \max_{r=1}^n \sum_{c=1, c \neq r}^{n} \left(|\mathbb{E}[\mathcal{E}_R\mathcal{E}_R^T](r,c)| + |\mathbb{E}[\mathcal{E}_R\mathcal{E}_I^T](r,c)| \right),\\\max_{r=1}^n \sum_{c=1, c \neq r}^{n} \left(|\mathbb{E}[\mathcal{E}_I\mathcal{E}_R^T](r,c)| + |\mathbb{E}[\mathcal{E}_I\mathcal{E}_I^T](r,c)| \right)
 \end{bmatrix} + \frac{3}{2L} \ \ \ \ (\because \text{ Using Eq.~\ref{def:C_bound}}).
 \end{align}
 \normalsize
Split $\mathcal{E}= \mathcal{Y}- \mathcal{X}W_i$ into their real and imaginary parts (subscripted by $R$ and $I$ respectively). Now $\mathcal{E}_R = \begin{bmatrix} \mathcal{Y}_R & -\mathcal{X}_R & \mathcal{X}_I\end{bmatrix}
 \begin{bmatrix} 1 & (W_i)_R & (W_i)_I \end{bmatrix}^T$, and $\mathcal{E}_I = \begin{bmatrix} \mathcal{Y}_I & -\mathcal{X}_I & -\mathcal{X}_R\end{bmatrix}
 \begin{bmatrix} 1 & (W_i)_R & (W_i)_I \end{bmatrix}^T$. For $r,c \in \{1,\cdots,n \}$, $r \neq c$, the correlation between the $r^{th}$ and $c^{th}$ sample is 
\small
\begin{align*}&\mathbb{E}[\mathcal{E}_R[r] \mathcal{E}_R[c]] = \nonumber\\
&\begin{bmatrix} 1 & (W_i^T)_R & (W_i^T)_I \end{bmatrix} \begin{bmatrix} \mathbb{E}(\mathcal{Y}_R[r]\mathcal{Y}_R[c]) & -\mathbb{E}(\mathcal{Y}_R[r]\mathcal{X}_R(c,:)) &\mathbb{E}(\mathcal{Y}_R[r]\mathcal{X}_I(c,:)) \\
-\mathbb{E}(\mathcal{X}_R(r,:)^T\mathcal{Y}_R[c]) &\mathbb{E}(\mathcal{X}_R(r,:)^T\mathcal{X}_R(c,:))
&-\mathbb{E}(\mathcal{X}_R(r,:)^T\mathcal{X}_I(c,:)) \\
\mathbb{E}(\mathcal{X}_I(r,:)^T\mathcal{Y}_R[c]) &-\mathbb{E}(\mathcal{X}_I(r,:)^T\mathcal{X}_R(c,:))
&\mathbb{E}(\mathcal{X}_I(r,:)^T\mathcal{X}_I(c,:)) \\
\end{bmatrix}\begin{bmatrix} 1 \\ (W_i)_R \\ (W_i)_I \end{bmatrix}.
\end{align*} 
\normalsize

From Eq.~\ref{eqn:RR_correlation} and Eq.~\ref{eqn:eq5_Lambda}, it follows that, $|\mathbb{E}[\mathcal{E}_R[r] \mathcal{E}_R[c]]| \leq \frac{U}{L} {\sqrt{27}}B^{rc}$. Similarly,
\begin{align}\label{eqn:relation1}
 &|\mathbb{E}[\mathcal{E}_R[r] \mathcal{E}_I[c]]| \leq \frac{U}{L} {\sqrt{27}}B^{rc}, |\mathbb{E}[\mathcal{E}_I[r] \mathcal{E}_R[c]]| \leq \frac{U}{L} {\sqrt{27}}B^{rc}, |\mathbb{E}[\mathcal{E}_I[r] \mathcal{E}_I[c]]| \leq \frac{U}{L} {\sqrt{27}}B^{rc}.
\end{align}
{
Using the inequalities Eq.~\ref{eqn:relation1} in Eq.~\ref{eqn:C1_upperbound}, we have $\|\mathcal{C}_1 \|_2 \leq \frac{3}{2L}+ 2{\sqrt{27}}\sum_{c=1,c\neq r}^n\frac{U}{L}B^{rc}$. Then, it follows from Eq.~\ref{eqn:gamma_sumBound1}, 
\small
\begin{align} \label{eqn:C1_upperbound2}
\|\mathcal{C}_1 \|_2 &\leq \frac{3}{2L}+ {6\sqrt{3}}\frac{U}{L}\frac{2C}{\delta -1}.
\end{align}
\normalsize}
\end{proof}

\end{document}